\documentclass[9pt]{article}   
\usepackage{tikz}
\usepackage{tikz-3dplot}
\usetikzlibrary{hobby}
\usetikzlibrary{positioning}
\usetikzlibrary{shapes}
\usepackage{graphicx}
\usepackage{color}
\usepackage{subfigure}
\usepackage{array}
\usepackage{longtable}

\usepackage{amsmath, amssymb, amsthm}
\usepackage{graphicx}
\usepackage{hyperref}
\usepackage{geometry}
\usepackage{cite}

\usepackage{authblk} 
\usepackage{abstract} 
\usepackage{enumitem} 

\newtheorem{theorem}{Theorem}[section]
\newtheorem{prop}{Proposition}[subsection]
\newtheorem{remark}{Remark}[subsection]
\newtheorem{example}{Example}[section]
\numberwithin{equation}{section}

\title{How Can Deep Neural Networks Fail Even With Global Optima?}
	
\author[]{Qingguang Guan \thanks{Correspondence to: \texttt{qingguang.guan@usm.edu}}}
\affil[]{School of Mathematics and Natural Sciences\\ 
	University of Southern Mississippi\\
	118 College Drive, Hattiesburg, MS, 39406}
	
\date{}
\begin{document}
	\maketitle
	
\begin{abstract}
	Fully connected deep neural networks are successfully applied to classification and function approximation problems. By minimizing the cost function, i.e., finding the proper weights and biases, models can be built for accurate predictions. The ideal optimization process can achieve global optima. However, do global optima always perform well? If not, how bad can it be? In this work, we aim to: 1) extend the expressive power of shallow neural networks to networks of any depth using a simple trick, 2) construct extremely overfitting deep neural networks that, despite having global optima, still fail to perform well on classification and function approximation problems. Different types of activation functions are considered, including ReLU, Parametric ReLU, and Sigmoid functions. Extensive theoretical analysis has been conducted, ranging from one-dimensional models to models of any dimensionality. Numerical results illustrate our theoretical findings. 
\end{abstract}
	


	\section{Introduction}
	Fully connected deep neural networks are the fundamental components of modern deep learning architectures, serving as the building blocks for various models like convolutional neural networks \cite{alexnet}, transformers \cite{17attention}, and numerous others. The effectiveness of deep neural networks lies in their ability to approximate complex functions, making them essential tools for tasks ranging from image recognition to natural language processing. However, along with their expressive power, deep neural networks also exhibit a phenomenon known as overfitting, where they may fit the training data very well instead of capturing the underlying patterns. This underscores the importance of understanding both the approximation capabilities and the limitations of deep neural networks.
	Since neural network models are obtained through training, which involves optimizing a cost function. The ultimate goal is to find the global optima, which represent configurations of the network parameters that minimize the discrepancy between the predicted outputs and the actual targets. However, achieving global optima does not guarantee optimal performance, as the network may still suffer from overfitting or other issues. Therefore, it is crucial to thoroughly examine the properties of global optima to understand how they affect the performance of the model.
	
	In this paper, we will focus on the regression problem formulated as scalar-valued function approximation. Let the target be a scalar-valued function $g({\bf x})$ (in the case of binary classification, $g({\bf x})$ has values $1$ and $-1$). The variable is ${\bf x} \in \mathbb{R}^d,$ where $d$ is a positive integer. The training set is defined as
	$$\Big\{{\bf x}_l,y_l\Big\}_{l=1}^{\mathbb{L}},$$ 
	where $y_l=g({\bf x}_l)$, and ${\bf x}_1, {\bf x}_2, ...,{\bf x}_\mathbb{L}$ are samples drawn from a uniform distribution in a $d$-dimensional cube $[0,1]^d$. Thus, the input layer has $d$ neurons, and the output layer has one neuron. Suppose there are $K$ hidden layers in the network. We define the output as a function $f_{K}({\bf x})$, omitting the parameters of weights and biases in the function definition. The cost functions are Median Absolute Error (MAE error) defined in equation \eqref{mae}  or Mean Squared Error (MSE error) defined in equation \eqref{mse}:
	\begin{align}
		&C_{mae}(W,B) = \frac{1}{\mathbb{L}}\sum_{l=1}^\mathbb{L} \Big|g({\bf x}_l)-f_K({\bf x}_l)\Big|,   \label{mae} \\
		&C_{mse}(W,B) = \frac{1}{\mathbb{L}}\sum_{l=1}^\mathbb{L} \Big(g({\bf x}_l)-f_K({\bf x}_l)\Big)^2, \label{mse} 
	\end{align}
	where $W,B$ are weights and biases of $f_K({\bf x})$. From \eqref{mae} and \eqref{mse}, we know
	$C_{mae}(W,B)\geq 0$ and $C_{mse}(W,B)\geq 0.$ If there exist $W^*$ and $B^*$ such that $C_{mae}(W^*,B^*) = 0$ or $C_{mse}(W^*,B^*) = 0,$ then $(W^*,B^*)$ is a global minimizer for the corresponding cost function.  
	
	For properly designed binary classification and function approximation problems, we can construct neural networks of any depth that fit the training data perfectly, achieving global optima and zero training loss. However, those neural networks have the worst generalization error. The extreme case is that the model only works on the training set; for any data not in the training set, the output of the model is meaningless.
	
	The paper is organized as follows: In Section \ref{trick}, we propose a simple trick to extend the universal approximation of shallow networks to deep neural networks of any depth. Various activation functions are considered. In Section \ref{sec3}, we construct examples of binary classification and function approximation in one, two, and high dimensions for networks with ReLU activation functions.
	Section \ref{sec4} is devoted to deep neural networks with Parametric ReLU activation functions. The constructions are slightly different compared to the ReLU function. In Section \ref{sec5}, we only consider function approximation problems for networks with Sigmoid activation functions. Conclusions are drawn in Section \ref{sec6}.
	
	\section{A Simple Trick to Extend the Expressivity of Shallow Neural Networks to Any Depth}\label{trick}
	
	The approximation properties of shallow neural networks have been extensively studied, including universal approximation \cite{cybenko1989approximation,park1991universal,hornik1993some, barron1993universal,WeinanEApriori,ma2022uniform}, and higher order estimations \cite{siegel2022high,siegel2022sharp}. However, extending these existing results to any depth is either too complicated or requires many neurons in the subsequent hidden layers, see \cite{hornik1991approximation,DY17,JX18,siegel2023optimal, kidger2020universal}. Before presenting examples that can cause deep neural networks (DNNs) to fail, we employ a very simple trick to extend the expressive power of shallow neural networks to networks of any depth, the minimum width of the following attached hidden layers can be as small as one. The activation functions are assumed to be ReLU-like \cite{nair2010rectified, clevert2015fast, klambauer2017self}, which have a linear part $x$, if $x\geq 0$; or $C^2$ continuous with bounded second order derivatives, such as Sigmoid, Tanh, Softplus \cite{glorot2011deep}, Gaussian and RBFs \cite{park1991universal}.
	
	\begin{theorem}\label{th2.1}
		Suppose a bounded scalar valued function $f({\bf x}), {\bf x}\in \mathbb{R}^d, d\geq 1$ can be approximated by a fully connected neural network with $\mathbb{K}$ hidden layers, $\mathbb{K}\geq 1$, then after attaching $\mathbb{N}$ extra hidden layers with any width $\geq 1$, the function can still be approximated by the deep neural network with $\mathbb{K}+\mathbb{N}$ hidden layers. $\mathbb{N}$ can be any positive integer.  	
	\end{theorem}
	\begin{proof}
		Let $f_\mathbb{K}$ be the approximation of $f$ obtained by a neural network with $\mathbb{K}$ hidden layers. Suppose that by increasing the number of neurons and adjusting weights and biases, we can achieve $f_\mathbb{K} \rightarrow f$.
		
		Next, we will develop a method to construct $f_{\mathbb{K}+\mathbb{N}}$ such that it also approximates the function $f$. Suppose the activation function $a(\cdot)$ has a bounded second-order derivative and $a'(c) \neq 0$, where $c$ is a constant. Then at layer $\mathbb{K}+n, n\geq 1$, we select one neuron, and let its input be
		\begin{align}
			&p_1 = \epsilon f_{\mathbb{K}} +c, \label{p1}\\
			&p_n = \frac{1}{a'(c)}a(p_{n-1})-\frac{a(c)}{a'(c)}+c, \ n\geq 2, \label{pN}
		\end{align}
		where $\epsilon>0$ is a small enough number, the inputs for other neurons at layer $\mathbb{K}+n$ are set to zeros.  The output of hidden layer $\mathbb{K}+\mathbb{N}$ is set to
		\begin{equation}\label{fkn}
			f_{\mathbb{K}+\mathbb{N}} = \frac{1}{a'(c)\epsilon}a(p_\mathbb{N})-\frac{a(c)}{a'(c)\epsilon},
		\end{equation}
		and we have the estimation
		\begin{equation}\label{est1}
			\left|f_{\mathbb{K}+\mathbb{N}} - f_\mathbb{K} \right|\leq C\epsilon,
		\end{equation}
		where $C>0$ doesn't depend on $\epsilon$. 
		
		Since $a(\cdot)$ is a $C^2$ continuous function and its second order derivative is bounded. To prove \eqref{est1}, from 
		\eqref{p1}-\eqref{pN}, by Taylor's expansion
		$$
		a(p_{n-1}) = a(c)+a'(c)(p_{n-1}-c) + \frac{a''(\xi_{n-1})}{2}(p_{n-1}-c)^2,
		$$
		where $\xi_{n-1}$ is a value between $p_{n-1}$ and $c$, we have
		\begin{align}\label{pnpn-1}
			p_n = p_{n-1} +\frac{a''(\xi_{n-1})}{2a'(c)}(p_{n-1}-c)^2, \ n\geq 2,
		\end{align}
		where  $\left|{a''(\xi_{n-1})}/{2a'(c)}\right|\leq M$. Subtracting $c$ from both sides, we obtain
		\begin{align}\label{pn-c}
			|p_n -c |&\leq |p_{n-1}-c| +M|p_{n-1}-c|^2,\ 2\leq n\leq \mathbb{N}
		\end{align}
		Summing \eqref{pn-c} for $n=2,3,\cdots$, we have
		\begin{align}\label{pn-c-1}
			|p_n -c | &\leq |p_{1}-c| + M\sum\limits_{i=1}^{n-1}|p_{i}-c|^2,\ 2\leq n\leq \mathbb{N}
		\end{align}
		From \eqref{pn-c-1}, $|p_1-c|\leq |f_\mathbb{K}|\epsilon$, $f_\mathbb{K}$ is bounded and $\mathbb{N}$ is finite, we obtain 
		\begin{align}\label{pn-c-2}
			|p_n -c | &\leq C_\mathbb{N} \epsilon,\ 2\leq n\leq \mathbb{N}
		\end{align}
		where $C_\mathbb{N}$ is a constant only depends on $M, \mathbb{N}$ and the bound of $f_\mathbb{K}$. Then from \eqref{pnpn-1}, we have
		\begin{align}\label{pnpn-2}
			p_\mathbb{N} = p_{1} +\sum\limits_{i=1}^{\mathbb{N}-1}\frac{a''(\xi_{i})}{2a'(c)}(p_{i}-c)^2.
		\end{align}
		So that by \eqref{pn-c-2} and \eqref{pnpn-2}, we have 
		$$
		p_\mathbb{N}=p_1+ O(\epsilon^2),
		$$
		from \eqref{p1}, \eqref{fkn} and Taylor's expansion of $a(p_{\mathbb{N}})$, we have 
		\begin{align*}
			f_{\mathbb{K}+\mathbb{N}} 
			&= (p_\mathbb{N}-c)/\epsilon +O(\epsilon)\\ 
			&= (p_\mathbb{N}-p_1)/\epsilon +(p_1-c)/\epsilon +O(\epsilon)\\
			&= f_\mathbb{K} +O(\epsilon)
		\end{align*}
		which verifies \eqref{est1}.
		
		For ReLU-like activation functions, at layer $\mathbb{K}+n, n\geq 1$, we also select one neuron, let its input be the same as \eqref{p1}-\eqref{pN},  where $c>0$ is larger enough such that $p_1>0$. Then equation \eqref{pN} becomes
		\begin{align*}
			p_n 
			&= c+\frac{a(p_{n-1})-a(c)}{a'(c)}\\
			&= c+(p_{n-1}-c)/1 \\
			&= p_{n-1}
		\end{align*}
		where $2\leq n\leq \mathbb{N}$. So that we have 
		$$
		f_{\mathbb{K}+\mathbb{N}} = f_\mathbb{K},
		$$
		which concludes the proof.
	\end{proof}

	\section{Examples for Networks with ReLU Activation Functions}\label{sec3}
	\subsection{One-Dimensional Examples}\label{1de}
	We start with a one-dimensional input variable $x \in [0,1]$, and a fully connected neural network consisting of two hidden layers with ReLU activation functions, and one output neuron with the linear activation function. We define the width of the first hidden layer as $H_1$ and the width of the second hidden layer as $H_2$.  This network can be represented as a function $f_2(x)$. The cost function can be Median Absolute Error or Mean Squared Error. For classification and approximation problems, we divide the interval [0,1] into $N-1$ intervals with equal distances and collect $x_i = (i-1)/(N-1)$, $i=1,2,\cdots,N$. The pairs $\{x_i,y_i\}_{i=1}^{N}$ form the training  set, where $y_i$ can be labels or function values. In the following sections, the lower bounds of $H_1$ and $H_2$ are the minimal widths for hidden layer 1 and hidden layer 2 in the neural networks.

	\subsubsection{One-Dimensional Binary Classification}\label{1dclass}
	
	We define the training set as $\{x_i,y_i\}_{i=1}^{N}$, where $y_i = -1$, if $x_i<0.5$, and $y_i=1$ if $x_i \geq 0.5.$ Our goal is to obtain a model capable of predicting the label for any given $x \in [0,1]$. Ideally, for any $x < 0.5$, the prediction should be $-1$, and for any $x \geq 0.5$, the prediction should be $1$. Let $\text{dis}(x,\{x_i\}_{i=1}^{N})$ denote the distance between $x$ and the set $\{x_i\}_{i=1}^{N}$.
	
	\begin{prop}
		If $H_1\geq 2N+1$ and $H_2\geq 2$, there exist weights and biases such that the loss function is zero, i.e., the optimal global minimum is achieved. Meanwhile, for any $\epsilon > 0$ small enough, if $\text{dis}(x,\{x_i\}_{i=1}^{N}) > \epsilon$, then $f_{2}(x) = 0$, indicating that the classification of $x$ is neither 1 nor -1. It lies on the decision boundary and cannot be determined.
	\end{prop}
	
	\begin{proof}
		Next, we demonstrate how to build these networks.  Let $h >0$, as shown in \cite{DY17,JX18}, at $x_i$, a basis function can be constructed 
		\begin{equation}\label{1dhat}
			\phi_i(x) = \frac{1}{h}a(x-x_i +h) -\frac{2}{h}a(x-x_i)+\frac{1}{h}a(x-x_i-h),
		\end{equation}
		where $a(x)$ is the ReLU activation function, $\phi_i(x)$ has the height $1$, and compact support $[x_i-h,x_i+h]$. Figure \ref{1d1basis} is an example of $\phi_i(x)$ with $x_i=0.5$ and $h=1/10$. 
		
		To reduce the number of neurons in the hidden layers, let $h=\frac{1}{2(N-1)}$, we use ${2N+1}$ neurons in the first hidden layer,  each neuron has distinct input as: $x-x_i$, $x-x_i-h$ or $x-x_i+h$, $i=1,2,\cdots,N$. Then, with the fact $x-x_{i+1}+h = x-x_i-h$, using the output of first hidden layer, we can construct basis functions $\phi_i(x)$, where $i=1,2,\cdots,N$, each possessing non-overlapping compact support. 
		
		For the second hidden layer, we use $2$ neurons, the input of first neuron is:
		\begin{equation}\label{1d1neuron}
			I_1(x) := \left(\sum_{i\ {\rm for} \ x_i\geq 0.5} \phi_i(x) \right)-b_1,
		\end{equation}
		where $b_1$ is the bias. It's easy to see that $I_1(x)$ is the result of a linear combination of the output of the first hidden layer with $b_1$ subtracted. Similarly, we have the input of the second neuron:
		\begin{equation}\label{1d2neuron}
			I_2(x) := \left(\sum_{i\ {\rm for} \ x_i< 0.5} \phi_i(x) \right) -b_2.
		\end{equation}  
		
		Let $b_1,b_2 \in [0,1)$, then the final output is:
		$$
		f_2(x) := \frac{a(I_1(x))}{1-b_1}+(-1)\frac{a(I_2(x))}{1-b_2},
		$$ 
		which concludes the proof.
	\end{proof}
	\begin{remark}
		Let $b = b_1 = b_2 \in [0,1)$. Then, the measure of compact support for $f_2(x)$ on $[0,1]$ is $1-b$. Additionally, $f_2(x_i)$, where $i=1,2,\cdots,N$, corresponds exactly to the correct label.
	\end{remark}
	
	For a fixed $N$, the weights and biases of the model represent a global optimal solution for both MSE and MAE cost functions. Consequently, there exist infinitely many global optimal solutions due to the variability of $b_1, b_2$. Nevertheless, the model will fail to predict any point's label if the point lies outside a small region around any $x_i$.
	
	Let $N$ be 6, we show the structure of the proposed network in Figure \ref{fig1_N6}, where the green cubes are intermediate values. Figure \ref{1d6basis_b_0} shows the graph of $f_2(x)$ when $b_1=b_2=0$, while Figure \ref{1d6basis_b_09} shows $f_2(x)$ when $b_1=b_2=0.9$. As $b_1, b_2$ approach 1, $f_2(x)$ develops ``spikes".
	\begin{remark}
		For the cross-entropy cost function
		$$
		-\frac{1}{N}\sum\limits_{i=1}^{N} y_i\log\Big(f_2(x_i)\Big)+(1-y_i)\log\Big(1-f_2(x_i)\Big),
		$$
		we can set the final output as
		$$
		f_2(x) := \frac12\frac{a(I_1(x))}{1-b_1}-\frac12\frac{a(I_2(x))}{1-b_2}+\frac12.
		$$ 
		Figure \ref{1d6basis_b_095_entropy} shows $f_2(x)$ when $N=6$, $b_1=b_2=0.95$. 
	\end{remark}
	
	\begin{figure} 
		\centering
		\includegraphics[width=1.0\textwidth]{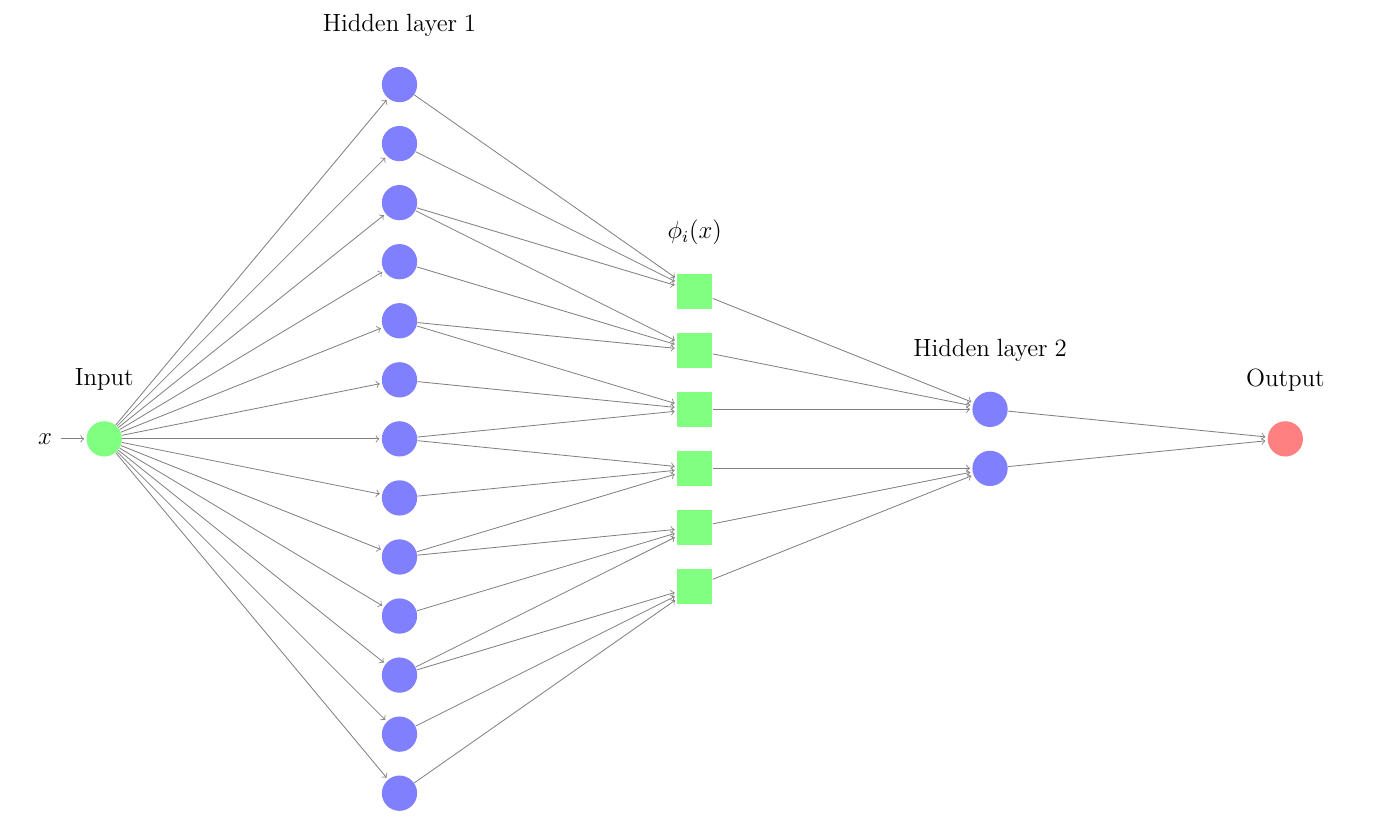}
		\caption{The network structure for 1-D binary classification when $N=6$}
		\label{fig1_N6}%
	\end{figure}
	
	\begin{figure} 
		\centering
		\subfigure[][]{%
			\label{1d1basis}%
			\includegraphics[width=0.35\linewidth]{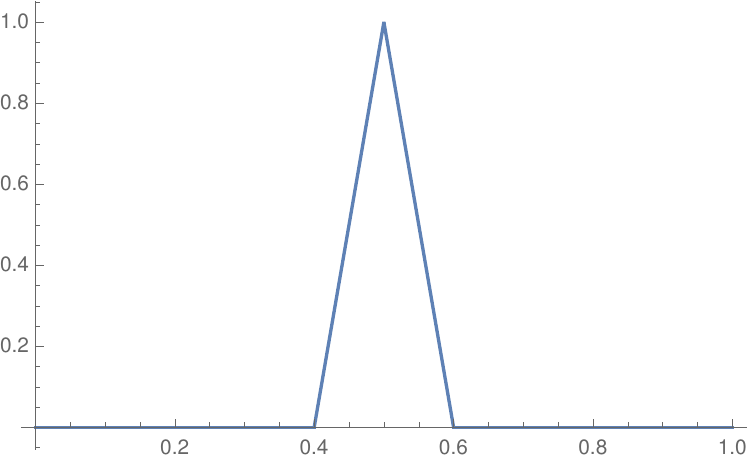}}%
		\hspace{8pt}%
		\subfigure[][]{%
			\label{1d6basis_b_0}%
			\includegraphics[width=0.35\linewidth]{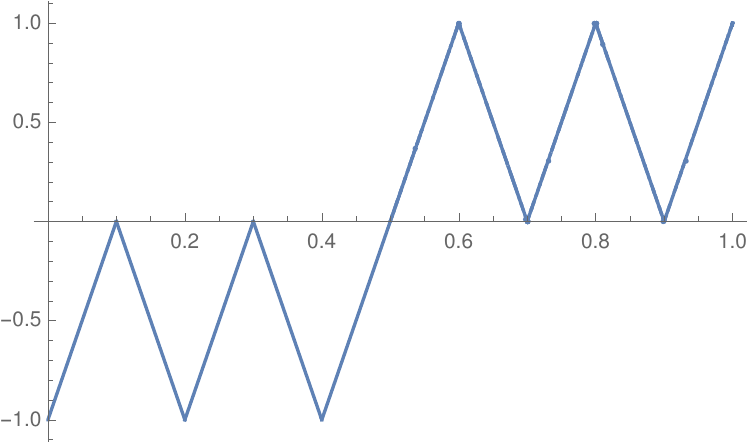}} \\
		\subfigure[][]{%
			\label{1d6basis_b_09}%
			\includegraphics[width=0.35\linewidth]{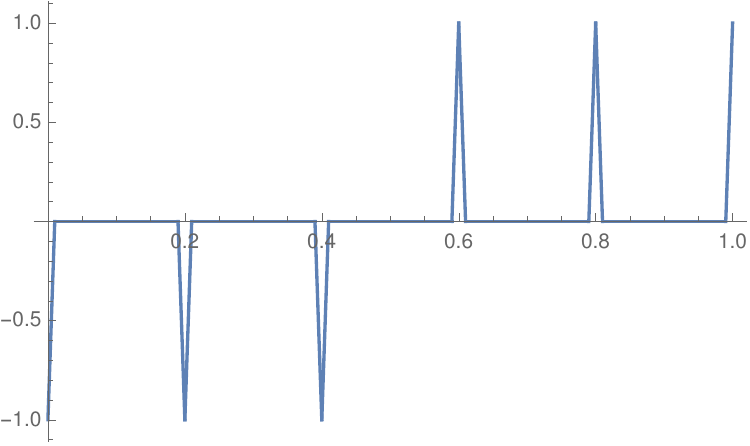}}%
		\hspace{8pt}%
		\subfigure[][]{%
			\label{1d6basis_b_095_entropy}%
			\includegraphics[width=0.35\linewidth]{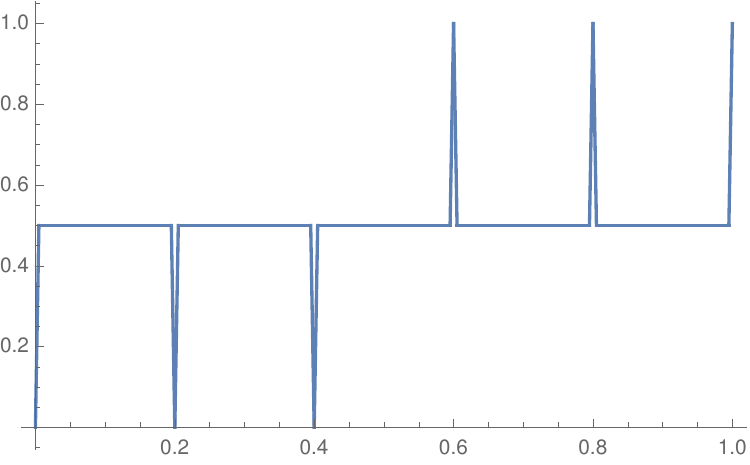}}%
		\caption[figures for 1-d]{
			\subref{1d1basis} The basis function $\phi_i(x)$, $x_i=0.5, h=0.1$;
			\subref{1d6basis_b_0} $f_2(x)$ when $N=6$, $b_1=b_2=0$;
			\subref{1d6basis_b_09} $f_2(x)$ when $N=6$, $b_1=b_2=0.9$; and,
			\subref{1d6basis_b_095_entropy} $f_2(x)$ for the cross-entropy cost function,  where $N=6$, $b_1=b_2=0.95$.}%
		\label{fig:1-d-1}%
	\end{figure}

	\subsubsection{One-Dimensional Function Approximation}\label{1dfunc}
	Suppose $g(x)$ is a continuous function for $x\in [0,1]$. To approximate $g(x)$ by the neural network, we define the training set $\{ x_i,y_i\}_{i=1}^{N},$ where $y_i=g(x_i)$. Ideally, $f_2(x)$ approaches $g(x)$ pointwise as the widths of the hidden layers increase. The specific ways to construct such networks are given in \cite{DY17,JX18}. However, if we go to another direction, even with global optima, the approximation can be very poor.
	
	\begin{prop}
		If $H_1\geq N+2$ and $H_2\geq N$,  there exist weights and biases such that the loss function is zero. Meanwhile, for any $\epsilon > 0$ small enough, if $\text{dis}(x,\{x_i\}_{i=1}^{N}) > \epsilon$, then $f_{2}(x) = 0$, indicating if $g(x)\not= 0$, the approximation is poor.
	\end{prop}
	\begin{proof}
		Let $h$ be $1/(N-1)$ and $\phi_i(x)$ be the same as \eqref{1dhat}.  
		For the first hidden layer, we employ $N+2$ neurons. Each neuron has a distinct input, such as $x-x_i$, $x-x_i-h$, or $x-x_i+h$, where $i=1,2,\cdots,N$. Then, with the fact that $x-x_{i+1} = x-x_i-h$, using the output of the first hidden layer, we can build basis functions $\phi_i(x)$, $i=1,2,\cdots,N$, which have overlapping compact support.
		Then we use $N$ neurons in second hidden layer and the final output is:
		$$
		f_2(x) := \sum_{i=1}^{N}y_i\frac{a(\phi_i(x)-b)}{1-b},
		$$ 
		where $b\in [0,1)$. So as $b\rightarrow 1$, the approximation will fail though it is the global optimal solution. 
	\end{proof}
	Until now, the networks have only two hidden layers, how about deeper ones, can we still make them fail? 
	The answer is yes, we discuss the construction in the following proposition.
	\begin{prop}\label{deeper}
		Since $g(x)$ is bounded, so is $f_2(x)$,  we can add $\mathbb{M}$ hidden layers with ReLU activation functions and apply equations \eqref{p1} -\eqref{fkn}, with $\epsilon>0$ and large enough $c$, we obtain
		$$
		f_{2+\mathbb{M}}(x) = f_{2}(x)
		$$
		where $\mathbb{M}$ is any positive integer and $f_{2+\mathbb{M}}(x)$ represents the output of the deeper neural network.
	\end{prop}
	
	But for one-dimensional problems, to cause neural networks to fail, the number of neurons needs to be even larger than the training data. Generally, this scenario is not encountered in practical use. In the next sections, we will explore how, even with a huge amount of high-dimensional input data and smaller-sized neural networks, failures can still occur. Proposition \ref{deeper} also works for higher dimensional binary classification and approximation problems.
	
	\subsection{Two-Dimensional Examples}
	In this section, we will build two dimensional ``basis functions" based on one dimensional ones. The input variable is $(x,y)\in \mathbb{R}^2$, the region is $[0,1]\times [0,1]$. We define the one dimensional basis function as:
	\begin{equation}\label{ghat}
		\phi(\xi) = \frac{1}{h}a(\xi+h) -\frac{2}{h}a(\xi)+\frac{1}{h}a(\xi-h),
	\end{equation}
	where $\xi \in [0,1], h>0$, $a(\xi)$ is the ReLU function. As in Section \ref{1de}, we have the training set data $\{(x_i,y_j),z_{i,j}\}_{i,j = 1}^N$,
	where $x_1=y_1=0, x_N=y_N=1,$ $x_i, y_j$ are uniformly distributed in $[0,1]$, $z_{i,j}$ can be labels or function values. Denote
	\begin{equation}\label{phi_ij}
		\phi_i(x)=\phi(x-x_i) \text{ and } \phi_j(y)=\phi(y-y_j).
	\end{equation}
	Then we can define the 2-D ``basis function" as:
	$$
	\Phi(x,y) = \sum_{ \{x_i\} }\phi_i(x)+\sum_{ \{y_j\}}\phi_j(y),
	$$
	where sets $\{x_i\}, \{y_j\}$ are chosen as needed, see Figure \ref{fig:2-d-1} for examples. Similar to Section \ref{1de}, we consider a fully connected neural network with two input neurons, $2+\mathbb{M}$ hidden layers with ReLU activation functions, and one output neuron with linear activation function. 
	We define the width of the first hidden layer as $H_1$, the width of the second hidden layer as $H_2$, and the width of the $(2+k)$th hidden layer as $H_{2+k}$, where $1 \leq k \leq \mathbb{M}$, $\mathbb{M}$ can be any positive integer. Additionally, we denote the output of the neural network as $f_{2+\mathbb{M}}(x,y)$.  The cost function options include MSE error or MAE error.
	
	\subsubsection{Two-Dimensional Binary Classification}
	In the training set $\{(x_i,y_j),z_{i,j}\}_{i,j = 1}^N$ for the binary classification problem, $z_{i,j}=-1$ if $x_i<0.5$, and $z_{i,j}=1$ if $x_i \geq 0.5$. The size of the training data is $N^2$. 
	
	\begin{prop}
		If $H_1\geq 4N+2$, $H_2\geq 2$,  and $H_{2+k}\geq 1$, $1 \leq k \leq \mathbb{M}$, there exist weights and biases such that the loss function is zero. Meanwhile, for any $\epsilon > 0$ small enough, if $\text{dis}((x,y),\{x_i,y_j\}_{i,j=1}^{N}) > \epsilon$, then $f_{2+\mathbb{M}}(x,y) = 0$.
	\end{prop}
	\begin{proof}
		Let $x_{i+1}-x_{i} = y_{j+1}-y_{j} = 2h>0$. Firstly, we construct the network with two hidden layers. For the first hidden layer, we need $2(2N+1)$ neurons. Each neuron has a distinct input as follows: $x-x_i$, $x-x_i-h$, or $x-x_i+h$; and $y-y_j$, $y-y_j-h$, or $y-y_j+h$, where $i,j=1,2,\cdots,N$. Then, we can use the output of the first hidden layer to construct the ``basis function" $\Phi(x,y)$ with a selected set $\{x_i,y_j\}$.
		
		For the second hidden layer, we need $2$ neurons, the input of the first neuron is:
		\begin{equation}\label{2d1neuron}
			I_1(x,y) := 
			\sum_{x_i\geq 0.5} \phi_i(x) 
			+
			\sum_{j=1}^{N}  \phi_j(y) 
			-b_1
		\end{equation}
		where $b_1$ is the bias. $I_1(x)$ is the linear combination of first hidden layer's output minus $b_1$. Similarly, we have the input of the second neuron:
		\begin{equation}\label{2d2neuron}
			I_2(x,y) := 
			\sum_{x_i< 0.5} \phi_i(x) 
			+
			\sum_{j=1}^{N}  \phi_j(y) 
			-b_2
		\end{equation}  
		Let $b_1,b_2 \in [1,2)$, then the final output for this network is:
		$$
		f_2(x,y) := \frac{a(I_1(x,y))}{2-b_1}+(-1)\frac{a(I_2(x,y))}{2-b_2}.
		$$ 
		If $b_1, b_2$ go to $2$, $f_2(x,y)$ develops ``spikes", see (a), (b) in Figure \ref{fig:2-d-2} for examples.
		
		Follow Proposition \ref{deeper}, we then construct the network with $2+\mathbb{M}$ hidden layers, such that 
		$f_{2+\mathbb{M}}(x,y) = f_2(x,y)$. 
	\end{proof}
	\begin{prop}
		Let $b=b_1=b_2\in [1,2),$ then the measure of compact support for $f_{2+\mathbb{M}}(x,y)$ on $[0,1]\times[0,1]$ is bounded by $N^2(2-b)^2/(N-1)^2$, which is decreasing to 0 as $b\rightarrow 2$.
	\end{prop}
	\begin{proof}
		For a certain point $(x_i,y_j), x_i\geq 0.5$, let's see how the compact support around it will shrink. The region to be considered is constrained to $[x_i-h,x_i+h]\times [y_j-h,y_j+h]$. On this region, it's easy to see, we have
		$$
		f_{2+\mathbb{M}}(x,y) = f_2(x,y) = \frac{a(\phi_i(x)+\phi_j(y) -b)}{2-b}. 
		$$
		Then the question becomes what's the area of compact support for ${a(\phi_i(x)+\phi_j(y) -b)}$, on which we have $\phi(x-x_i)+\phi(y-y_j) -b \geq 0$. Since $\phi(x-x_i), \phi(y-y_j)\leq 1$, we can remove the region on which $\phi(x-x_i)+1-b\leq 0$ or $\phi(y-y_j)+1-b\leq 0$. Then the compact support of $f_2(x,y)$ around $(x_i,y_j)$ is contained in a region $$[x_i-(2-b)h,x_i+(2-b)h]\times[y_j-(2-b)h,y_j+(2-b)h]$$ with area less than $(2(2-b)h)^2$. The analysis here can be applied to calculate the upper bound of high dimensional compact support's volume around a certain point. The total area on which the model $f_{2+\mathbb{M}}(x,y)$ can give us some positive feedback is less than 
		$4N^2(2-b)^2h^2$ where $h=1/(2N-2).$ The upper bound of the model's accuracy can be shown in Figures \ref{fig:2d3a} and \ref{fig:2d3b}.
	\end{proof}
	As the dimension increases, with a fixed number of sub-intervals in each dimension, we have more training data. Let $\mathbb{M}=1$, then the network has three hidden layers. We can compare the number of neurons and the size of training set in the following table:
	\begin{center}
		\begin{tabular}{l*{3}{c}}
			$N$                                      & 50    & 100    & 1000   \\
			\hline
			Data Set size $=N^2$                    & 2500 & 10000 & 1000000   \\
			First hidden layer                 & 202   & 402    & 4002    \\
			Second hidden layer            & 2       & 2       & 2 \\
			Third hidden layer            & 1       & 1       & 1 
		\end{tabular}
	\end{center}
	So even the size of training set is larger than the number of neurons, the model can still fail with the global optimal solution we constructed. Also more layers do not help.
	
	If we increase the number of neurons in the second hidden layer, then using the basic shapes shown in Figure \ref{fig:2-d-1}, we can build a variety of examples not only for binary classification but also for multi-class classification.

	\subsubsection{Two-Dimensional Function approximation}
	To approximate a continuous function $g(x,y)$, $(x,y)\in [0,1]^2$ by the ReLU neural network with $2+\mathbb{M}$ hidden layers, we use the training set $\{(x_i,y_j),z_{i,j}\}_{i,j=1}^N,$ where $z_{i,j} = g(x_i,y_j)$. $x_i,y_j \in [0,1]$ are uniformly distributed , $x_1=y_1 = 0, x_N=y_N=1$. 
	
	\begin{prop}
		If $H_1\geq 2N+4$, $H_2\geq N^2$  and $H_{2+k}\geq 1$, $1 \leq k \leq \mathbb{M}$,  there exist weights and biases such that the loss function is zero. Meanwhile, for any $\epsilon > 0$ small enough, if $\text{dis}((x,y),\{x_i,y_j\}_{i,j=1}^{N}) > \epsilon$, then $f_{2+\mathbb{M}}(x,y) = 0$, indicating if $g(x,y)\not= 0$, the approximation is poor.
	\end{prop}
	\begin{proof}
		Let $h = x_{i+1}-x_{i} = y_{j+1}-y_{j}$, we define the ``basis function" as: 
		\begin{equation}\label{2dhat2}
			\Phi_{i,j}(x,y) = \phi_i(x)+\phi_j(y),
		\end{equation}
		where $\phi_i, \phi_j$ are defined in \eqref{phi_ij}, $i,j=1,2,\cdots,N$, and $\Phi_{i,j}(x,y)$ has the height $2$, see Figure \ref{fig:2d1a} for an example. The inputs for first hidden layer are similar as classification problem, however, we only need $2(N+2)$ neurons, which is less. Using the outputs of the first hidden layer, we can build ``basis functions": $\Phi_{i,j}(x,y)$. 
		Then we need $N^2$ neurons in second hidden layer, which is much more, and the final output is:
		$$
		f_2(x,y) := \sum_{i,j=1}^{N}z_{i,j}\frac{a(\Phi_{i,j}(x,y)-b)}{2-b},
		$$ 
		where $b\in [1,2)$. So as $b\rightarrow 2$, the approximation will fail though it is the global optimal solution. Deeper networks will fail if constructed in the same way as described in Proposition \ref{deeper}.
	\end{proof}
	\begin{figure}%
		\centering
		\subfigure[][]{%
			\label{fig:2d1a}%
			\includegraphics[width=0.4\linewidth]{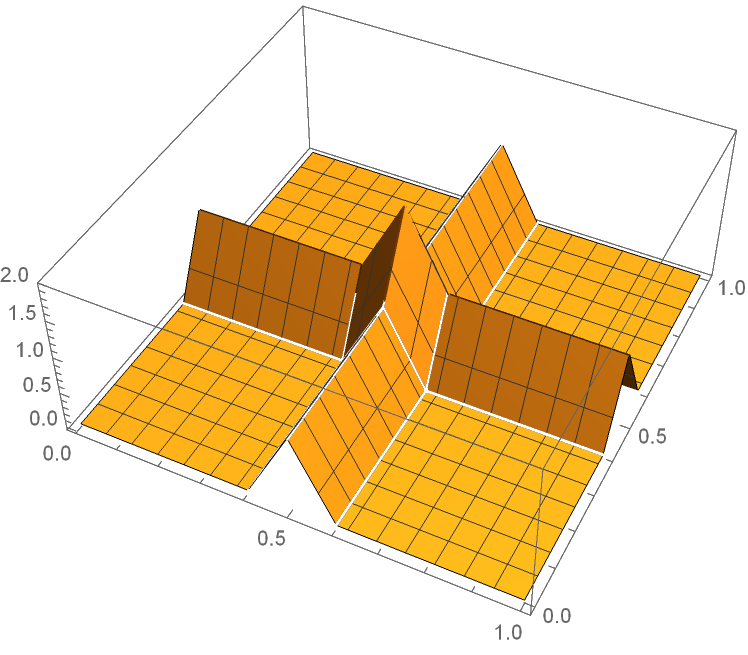}}%
		\hspace{8pt}%
		\subfigure[][]{%
			\label{fig:2d1b}%
			\includegraphics[width=0.4\linewidth]{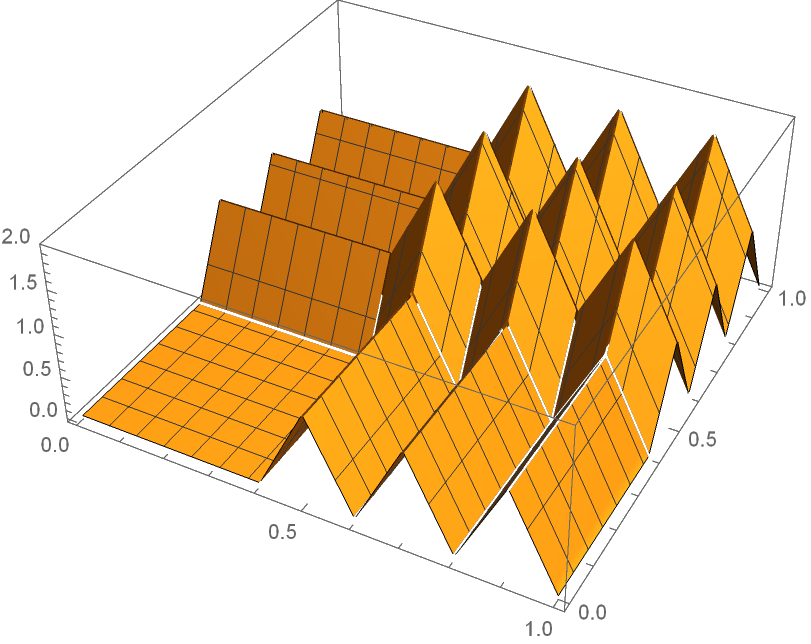}} \\
		\subfigure[][]{%
			\label{fig:2d1c}%
			\includegraphics[width=0.4\linewidth]{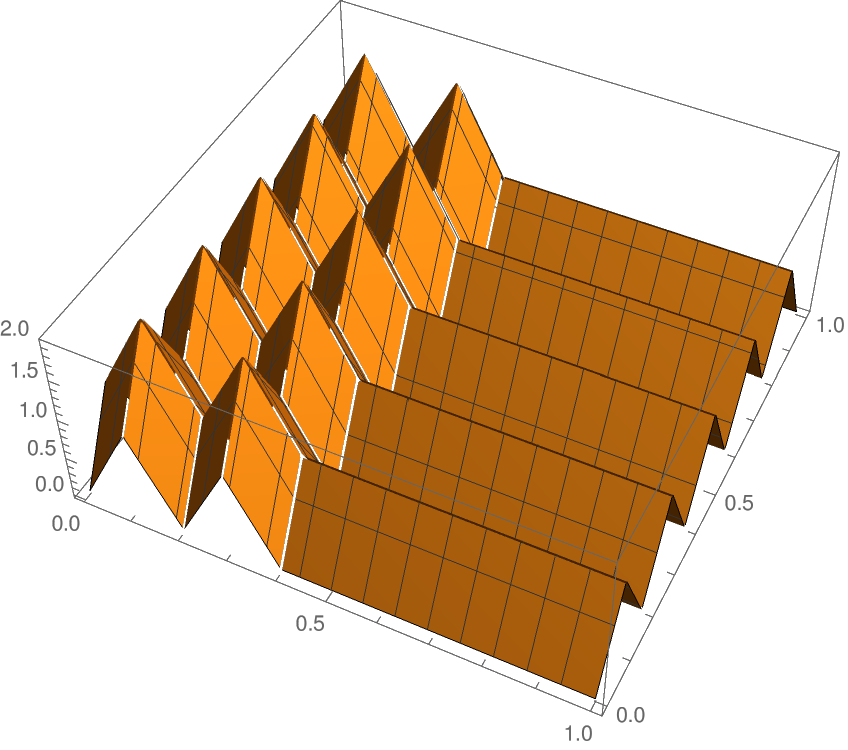}}%
		\hspace{8pt}%
		\subfigure[][]{%
			\label{fig:2d1d}%
			\includegraphics[width=0.4\linewidth]{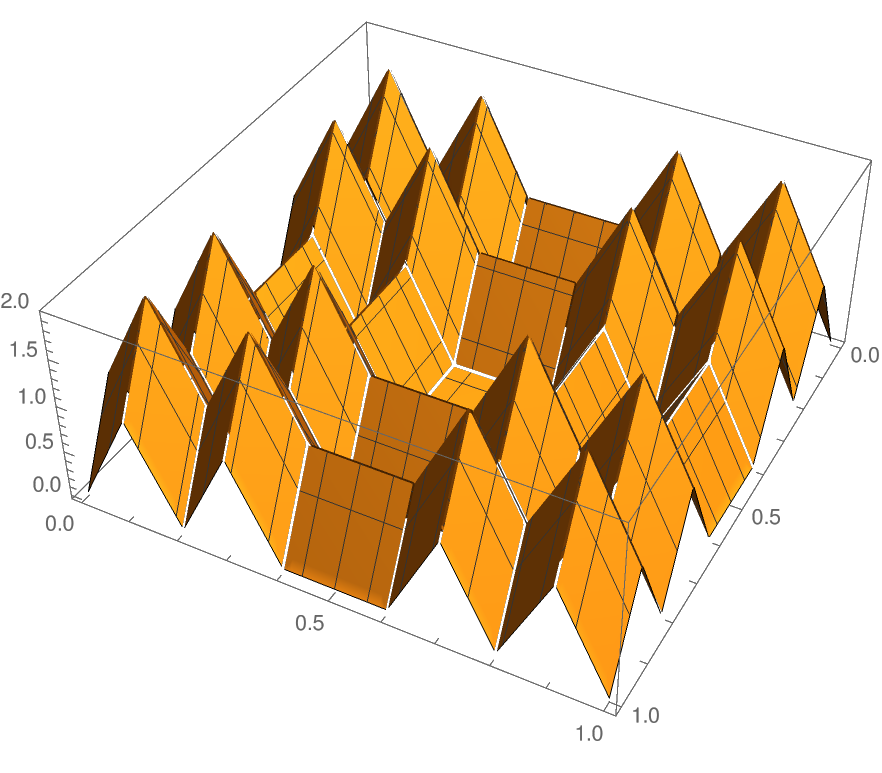}}%
		\caption[2d ``basis"]{
			\subref{fig:2d1a} 2-D ``basis function" with $x_i=y_j=0.5$ and $h=0.1$;
			\subref{fig:2d1b} 2-D ``basis function" with $x_i=y_j=0.5,0.7,0.9$ and $h=0.1$;
			\subref{fig:2d1c} 2-D ``basis function" with $x_i=0.1, 0.3$, $y_j=0.1, 0.3, 0.5, 0.7, 0.9$ and $h=0.1$; 
			and,
			\subref{fig:2d1d} 2-D ``basis function" with $x_i=0.1, 0.3, 0.7,0.9$, $y_j=0.1, 0.3, 0.7, 0.9$ and $h=0.1$.}%
		\label{fig:2-d-1}%
	\end{figure}
	\begin{figure}%
		\centering
		\subfigure[][]{%
			\label{fig:2d2a}%
			\includegraphics[width=0.4\linewidth]{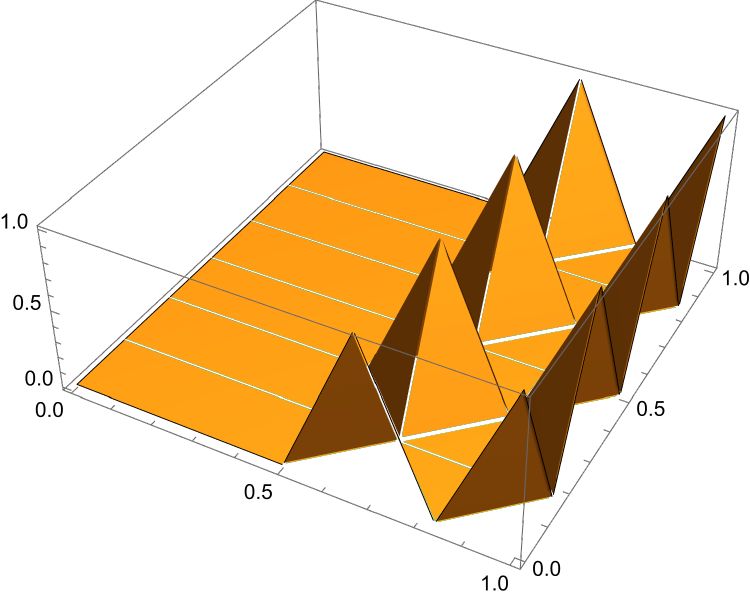}}%
		\hspace{8pt}%
		\subfigure[][]{%
			\label{fig:2d2b}%
			\includegraphics[width=0.4\linewidth]{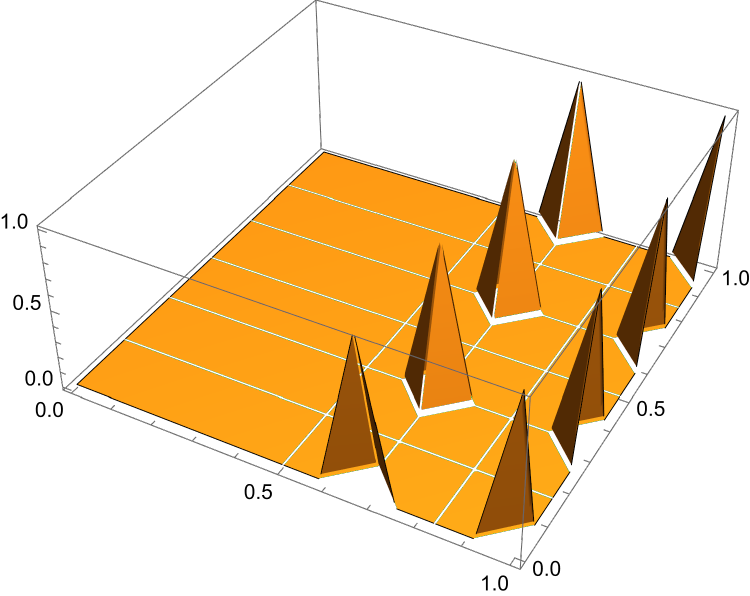}} 
		\caption[2d ``basis"]{ Here $N=4$, $x_1=y_1=0$, $x_N=y_N=1$, $h=1/6=(x_{i+1}-x_{i})/2$. 
			\subref{fig:2d2a} 
			Graph of $a(I_1(x,y))/(2-b_1)$ with  $b_1=1$;
			\subref{fig:2d2b} 
			Graph of  $a(I_1(x,y))/(2-b_1)$ with  $b_1=1.5$.
		}%
		\label{fig:2-d-2}%
	\end{figure}
	\begin{figure}%
		\centering
		\subfigure[][]{%
			\label{fig:2d3a}%
			\includegraphics[width=0.4\linewidth]{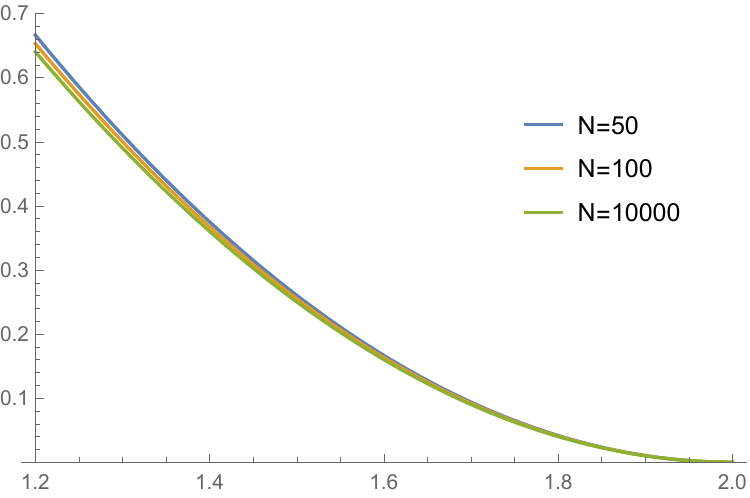}}%
		\hspace{8pt}%
		\subfigure[][]{%
			\label{fig:2d3b}%
			\includegraphics[width=0.4\linewidth]{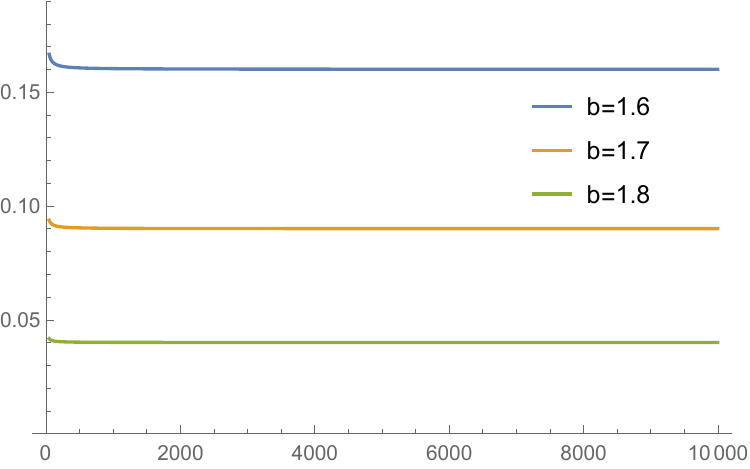}} 
		\caption[2d ``basis"]{ Graph for upper bound of model's accuracy $N^2(2-b)^2/(N-1)^2$.
			\subref{fig:2d3a} 
			For fixed $N$, let $b$ vary from $1.2$ to $2$;
			\subref{fig:2d3b} 
			for fixed $b$, let $N$ vary from $50$ to $10000$.
		}%
		\label{fig:2-d-3}%
	\end{figure}

	\subsection{High Dimensional Model Problems}\label{hd-relu}
	In this section, the examples will be generalized to high dimension.   
	Let ${\bf x } \in \mathbb{R}^d, d\geq 3,$
	$$
	{\bf x} = (x_1,x_2,\cdots,x_d),
	$$ 
	the region is a $d$ dimensional hyper-cube $[0,1]^d$. And $\{x_{1,i_1}\}_{i_1=1}^N$ is denoted as the set of scalar values, which are uniformly distributed in the first dimension within the interval $[0,1]$. Similarly, we have the set in the $j$th dimension within the interval $[0,1]$ as $\{x_{j,i_j}\}_{i_j=1}^N$, where $j=2,\cdots,d$. 
	The training set is defined as
	\begin{equation}\label{Txy}
		\Xi :=\bigg\{\big(x_{1,i_1},x_{2,i_2},\cdots,x_{d,i_d}\big),\ y_{i_1,i_2,\cdots,i_d}\bigg\}_{i_1,i_2,\cdots,i_d=1}^N
	\end{equation}
	where $x_{1,1}=x_{2,1}= \cdots =x_{d,1} =0$,
	$x_{1,N}=x_{2,N}= \cdots =x_{d,N} =1$, $y_{i_1,i_2,\cdots,i_d}$ can be labels or function values. 
	We denote the basis function in $j$th dimension, centered at $x_{j,i_j}$, as
	\begin{equation}\label{phi_j_ij}
		\phi_{j,i_j}(x_j)=\phi(x_j-x_{j,i_j}),
	\end{equation}
	where $x_j$ is the $j$th variable of ${\bf x}$.
	Then we can define the $d$-dimensional ``basis function" as:
	$$
	\Phi({\bf x}) = \sum_{j=1}^{d} 
	\left(
	\sum_{ \{x_{j,i_j}\}}\phi_{j,i_j}(x_j)
	\right),
	$$
	where sets $\{x_{j,i_j}\}$ are chosen as needed.
	We consider a fully connected neural network comprising $d$ input neurons, $2+\mathbb{M}$ hidden layers employing ReLU activation functions, and one output neuron with a linear activation function. The width of the first hidden layer is denoted as $H_1$, the width of the second hidden layer as $H_2$, and the width of the $(2+k)$th hidden layer as $H_{2+k}$, where $1 \leq k \leq \mathbb{M}$, and $\mathbb{M}$ represents any positive integer. Furthermore, we let the output of the neural network be $f_{2+\mathbb{M}}(x,y)$.  The cost function can be MSE or MAE errors.
	\subsubsection{High Dimensional Binary Classification}\label{hdbc}
	Suppose in the training set $\Xi$, see \eqref{Txy}, $y_{i_1,i_2,\cdots,i_d} = -1$ if $x_{1,i_1} < 0.5$, and $y_{i_1,i_2,\cdots,i_d} = 1$ if $x_{1,i_1} \geq 0.5$. The size of the training set is $N^d$. We define $\text{dis}({\bf x},\Xi_{\bf x})$ as the distance between ${\bf x} \in \mathbb{R}^d$ and the set $\Xi_{\bf x}$ from $\Xi$, which is 
	\begin{equation}\label{Tx}
		\Xi_{\bf x} :=\bigg\{\big(x_{1,i_1},x_{2,i_2},\cdots,x_{d,i_d}\big)\bigg\}_{i_1,i_2,\cdots,i_d=1}^N.
	\end{equation}
	\begin{prop}
		If $H_1\geq 2dN+d$, $H_2\geq 2$,  and $H_{2+k}\geq 1$, $1 \leq k \leq \mathbb{M}$, there exist weights and biases such that the loss function is zero. Meanwhile, for any $\epsilon > 0$ small enough, if $\text{dis}({\bf x},\Xi_{\bf x}) > \epsilon$, then $f_{2+\mathbb{M}}({\bf x}) = 0$.
	\end{prop}
	\begin{proof} Let $x_{j,i_j+1}-x_{j,i_j} = 2h>0, j=1,2,\cdots, d$. For the first hidden layer, we require $d(2N+1)$ neurons. Each neuron has distinct inputs: $x_j-x_{j,i_j}$, $x_j-x_{j,i_j}-h$, or $x_j-x_{j,i_j}+h$. Then, we can use the output of first hidden layer, to build ``basis function" $\Phi({\bf x})$ with a selected sets $\{x_{j,i_j}\}, j=1,2,\cdots, d$. 
		
		For the second hidden layer, we need $2$ neurons, the input of first neuron is:
		\begin{equation}\label{dd1neuron}
			I_1({\bf x}) := 
			\sum_{x_{1,i_1}\geq 0.5} \phi_{1,i_1}
			+
			\sum_{j=2}^{d} 
			\sum_{ i_j=1}^{N}\phi_{j,i_j}
			-b_1
		\end{equation}
		where $b_1$ is the bias. $I_1(x)$ is the linear combination of first hidden layer's output. Similarly, we have the input of second neuron:
		\begin{equation}\label{dd2neuron}
			I_2({\bf x}) := 
			\sum_{x_{1,i_1}< 0.5} \phi_{1,i_1}
			+
			\sum_{j=2}^{d} 
			\sum_{ i_j=1}^{N}\phi_{j,i_j}
			-b_2
		\end{equation}  
		Let $b=b_1=b_2 \in [d-1,d)$, and the final output be:
		$$
		f_{2}({\bf x}) := 
		\frac{a(I_1({\bf x}))}{d-b}
		+(-1)\frac{a(I_2({\bf x}))}{d-b}.
		$$ 
		The measure of compact support for $f_2({\bf x})$ on $[0,1]^d$ is decreasing to 0 as $b\rightarrow d$. 
		If we add $\mathbb{M}$ extra hidden layers, as in Proposition \ref{deeper}, the function $f_{2+\mathbb{M}}({\bf x}) = f_2({\bf x})$ can be constructed.
	\end{proof}
	
	Similar to the 2-D case, we know the compact support of $f_{2+\mathbb{M}}({\bf x})$ around $(x_{1,i_1}$,$x_{2,i_2}$,$\cdots$, $x_{d,i_d})$ is contained in a region with $d$-volume less than $(2(d-b)h)^d$, where $b=b_1=b_2 \in [d-1,d)$.  
	\begin{remark}
		The total $d$-volume on which the model $f_{2+\mathbb{M}}({\bf x})$ can give us some positive feedback is less than 
		$$\frac{N^d}{(N-1)^d}(d-b)^d.$$ 
	\end{remark}
	An example for 3-D case can be seen in Figure \ref{fig:3-d-1}. As $b$ approaches $d$ from $d-1$, a higher-dimensional problem demonstrates analogous behaviors.
	\begin{figure}
		\centering
		\subfigure[][]{%
			\label{fig:3d3a}%
			\includegraphics[width=0.4\linewidth]{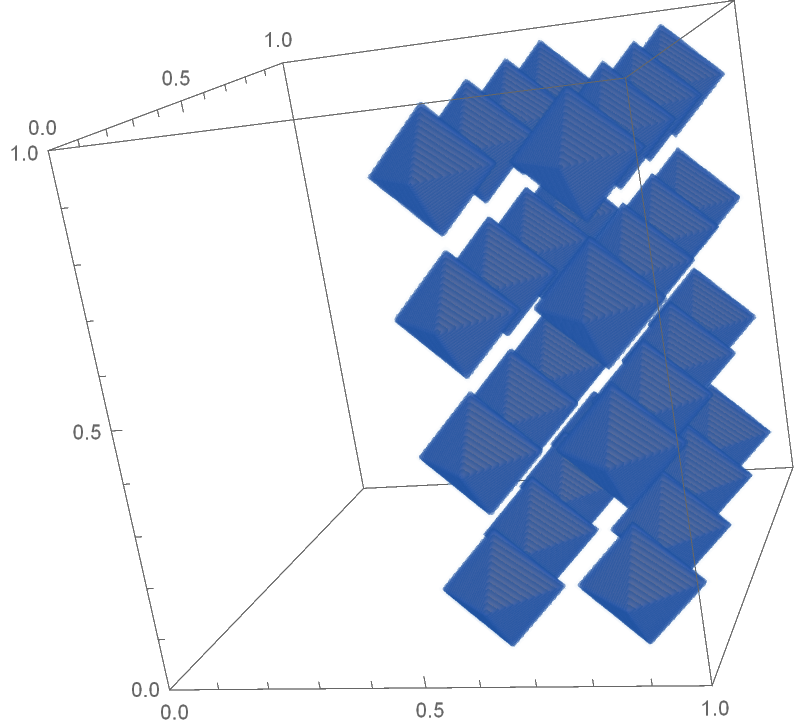}}%
		\hspace{8pt}%
		\subfigure[][]{%
			\label{fig:3d3b}%
			\includegraphics[width=0.4\linewidth]{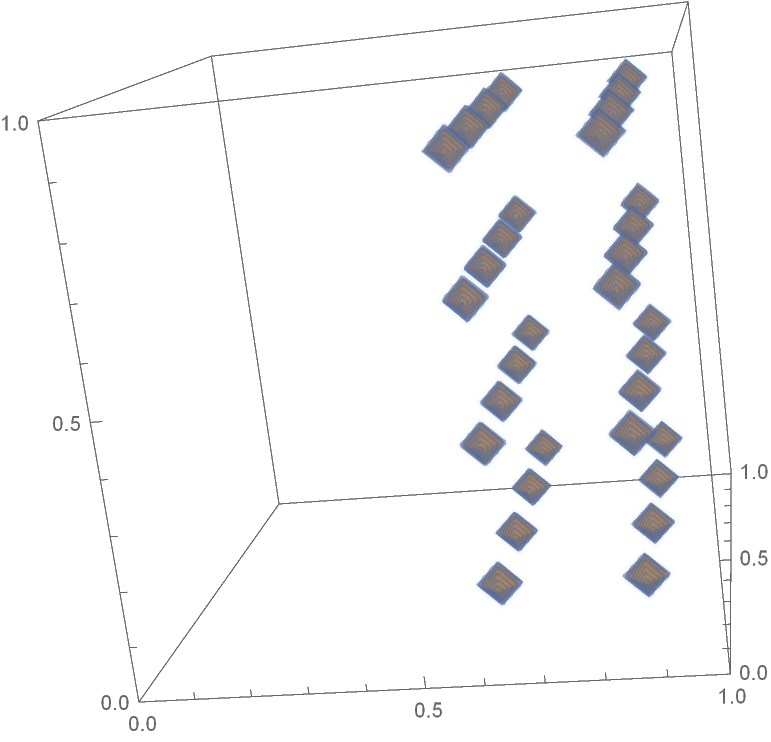}} 
		\caption[2d ``basis"]{Density color plots for a 3-D function: blank means 0, blue indicates values close to zero. The center of each diamond has a value of $1$, while the surface has a value of $0$. Here $N=4$, $x_{j,1} =1/8, j=1,2,3$, $h=1/8$,
			\subref{fig:3d3a} 
			Graph of $a(I_1({\bf x}))/(3-b)$ with  $b=2.1$;
			\subref{fig:3d3b} 
			Graph of $a(I_1({\bf x}))/(3-b)$ with $b=2.7$.
		}%
		\label{fig:3-d-1}%
	\end{figure}
	
	\subsubsection{Image classification}\label{sec_img}
	This is another example of a high-dimensional classification problem. Suppose we have $3\times 3$ pixel images, which can be viewed as $3\times 3$ matrices or $9$-dimensional vectors, with each pixel's grayscale ranging from $0$ to $255$. These images are divided into two classes. The `{\bf dark}' class comprises pixels with values less than $128$, while the `{\bf light}' class consists of pixels with values greater than or equal to $128$, as shown in Figure \ref{fig:3-d-2}. This serves as the ground truth.
	
	\begin{figure} 
		\centering
		\subfigure[][]{%
			\label{fig:3d2a}%
			\includegraphics[width=0.4\linewidth]{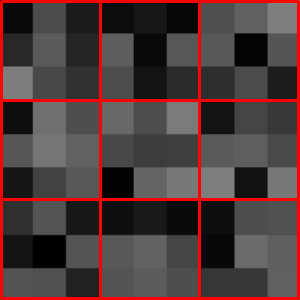}}%
		\hspace{8pt}%
		\subfigure[][]{%
			\label{fig:3d2b}%
			\includegraphics[width=0.4\linewidth]{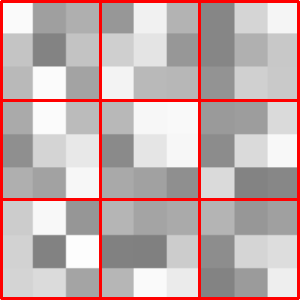}} 
		\caption[2d ``basis"]{
			\subref{fig:3d2a} 
			Nine `dark' pictures each has 3$\times$3 pixels;
			\subref{fig:3d2b} 
			Nine `light' ones each has 3$\times$3 pixels.
		}%
		\label{fig:3-d-2}%
	\end{figure}
	
	Usually, we build a model to predict classes with less training data compared with the ground truth. To choose the training data, we employ the color reduction as a transform of each pixel's gray scale:  
	\begin{equation}\label{grayscale}
		\left \lfloor{\frac{p}{m}}\right \rfloor 
		\times m, \quad m=1,2,\cdots, 255,
	\end{equation}
	where $0\leq p\leq 255$ represents the value of a given pixel, the $\left \lfloor{\ }\right \rfloor $ symbol denotes the floor function. The color and texture of images undergo unnoticeable changes if $m \leq 5$. Based on this observation, let $2\leq m \leq 5$. The training set is defined as follows: 
	$$\Xi := \bigg\{\big(x_{1,i_1},x_{2,i_2},\cdots,x_{9,i_9}\big), y_{i_1,i_2,\cdots,i_9}\bigg\}_{i_1,i_2,\cdots,i_9=1}^N$$
	where $N=\left \lfloor{255/m}\right \rfloor +1$,  $x_{j,i_j}=m (i_j-1), i_j=1,2,\cdots,N, j=1,2,\cdots,9.$ The `dark' set is defined as
	$$y_{i_1,i_2,\cdots,i_9} = -1, {\text{ if }} x_{1,i_1},x_{2,i_2},\cdots,x_{9,i_9}<128.$$
	The `light' set is defined as
	$$y_{i_1,i_2,\cdots,i_9} = 1, {\text{ if }} x_{1,i_1},x_{2,i_2},\cdots,x_{9,i_9}\geq 128.$$ 
	Combining the `dark' and `light' sets, we have the training set, which can be normalized by dividing each element $x_{j,i_j}$ by 255. The structure of the neural network is the same as in Section \ref{hd-relu}.
	
	\begin{prop}
		If $H_1\geq 18N+9$, $H_2\geq 2$,  and $H_{2+k}\geq 1$, $1 \leq k \leq \mathbb{M}$, there exist weights and biases such that the loss function is zero. Meanwhile, for any picture with $3\times 3$ pixels, if it's not in the training set, then it can't be classified.
	\end{prop}
	\begin{proof}
		Let $2h = m/255 = x_{j,i_j+1}-x_{j,i_j},$ where $2\leq m\leq 5$, for the first hidden layer, we need $9(2N+1)$ neurons, and the inputs are similar to those in Section \ref{hdbc}.
		For the second hidden layer, we only need 2 neurons; however, the inputs differ from those in Section \ref{hdbc}.
		The input of the first neuron is:
		\begin{equation}\label{dd1neurons}
			I_1({\bf x}) := 
			\sum_{j=1}^{9} 
			\sum_{x_{j,i_j}< 0.5}^{N}\phi_{j,i_j}
			-b_1
		\end{equation}
		where $b_1$ is the bias. $I_1(x)$ is the linear combination of the output from the first hidden layer. Similarly, let $b_2$ be the bias, we have the input for the second neuron:
		\begin{equation}\label{dd2neurons}
			I_2({\bf x}) := 
			\sum_{j=1}^{9} 
			\sum_{x_{j,i_j}\geq 0.5}^{N}\phi_{j,i_j}
			-b_2
		\end{equation}  
		Let $b=b_1=b_2 \in [8,9)$, then with $\mathbb{M}$ more hidden layers, the final output can be:
		$$
		f_{2+\mathbb{M}}({\bf x}) := 
		\frac{a(I_1({\bf x}))}{9-b}
		+(-1)\frac{a(I_2({\bf x}))}{9-b}.
		$$ 
		For $b$ close enough to 9, the model can only recognize images within the training set, even if the images differ by only one pixel with a one-unit grayscale variation.
	\end{proof}
	\begin{remark}
		The number of all `dark' and `light' images is $2\times 128^9$, so even when considering the training set, for $m=2$, the accuracy of the constructed model is as low as $64^9/128^9 = 1/512.$
	\end{remark}
	For simplicity, assume the network has just two hidden layers. We can compare the size of the training set and the number of neurons for $m=2, 5, 10$, as described in \eqref{grayscale}, in the following table:
	\begin{center}
		\begin{tabular}{l*{3}{c}}
			$m$                   & 2    & 5    & 10   \\
			\hline
			Training Set size         & $2\times 64^9$  &$2\times 26^9$  & $2\times 13^9$   \\
			First hidden layer    & 2313                   & 945      & 477   \\
			Second hidden layer   & 2                      & 2       & 2 
		\end{tabular}
	\end{center}
	The dataset is huge, and the number of neurons is relatively small. However, our model will still fail to reach acceptable accuracy, even with the global optimal solution and more hidden layers.

	\subsubsection{High Dimensional Function approximation}
	To approximate a continuous function $g({\bf x})$, where ${\bf x}\in [0,1]^d$, using a ReLU neural network with $2+\mathbb{M}$ hidden layers, we define the training set $\Xi$ as in \eqref{Txy}, 
	where $y_{i_1,i_2,\cdots,i_d} = 
	g(x_{1,i_1},x_{2,i_2},\cdots, x_{d,i_d})$, and $x_{j,1}=0,  x_{j,N}=1$, $j=1,2,\cdots,d$. The values  $x_{j,i_j}\in [0,1]$ are distributed uniformly. The cost functions are MSE  or MAE errors. Using the notations in Section \ref{hdbc}, we have the following proposition.
	\begin{prop}
		If $H_1\geq dN+2d$, $H_2\geq N^d$,  and $H_{2+k}\geq 1$, $1 \leq k \leq \mathbb{M}$, there exist weights and biases such that the loss function is zero. Meanwhile, for any $\epsilon > 0$ small enough, if $\text{dis}({\bf x},\Xi_{\bf x}) > \epsilon$, then $f_{2+\mathbb{M}}({\bf x}) = 0$,  which means if $g({\bf x})\not= 0$, the approximation is poor.
	\end{prop}
	\begin{proof}
		Let $h$ be $x_{j,i_j+1}-x_{j,i_j}$. We define the ``basis function" as: 
		\begin{equation}\label{ddhat2}
			\Phi_{i_1,i_2,\cdots,i_d}({\bf x}) = \sum_{j=1}^{d} 
			\phi_{j,i_j},
		\end{equation}
		where $\phi_{j,i_j}$ is introduced in \eqref{phi_j_ij}, and $\Phi_{i_1,i_2,\cdots,i_d}({\bf x})$ has a maximum height of $d$.
		
		The inputs for the first hidden layer are similar to those in the binary classification problem. However, we only need $d(N+2)$ neurons, which is fewer. 
		Then we need $N^d$ neurons in the second hidden layer, which is significantly more, and the final output would be:
		$$
		f_{2+\mathbb{M}}({\bf x}) := \sum_{i_1,i_2,\cdots,i_d=1}^{N}y_{i_1,i_2,\cdots,i_d}
		\frac{a(\Phi_{i_1,i_2,\cdots,i_d}({\bf x})-b)}
		{d-b},
		$$ 
		where $b\in [d-1,d)$. As $b$ approaches $d$, the approximation will fail even though $f_{2+\mathbb{M}}({\bf x})$ represents the global optimal solution. 
	\end{proof}
	
	\section{Constructions for Networks with Parametric ReLU Activation Functions}\label{sec4}
	In this section, we will show how to construct the global optimal solutions for networks with Parametric ReLU  activation functions, which are proposed in \cite{he2015delving}. The cost functions can be MSE or MAE errors. The Parametric ReLU activation function is denoted as:
	\begin{equation}\label{prelu}
		\sigma (x)  = 
		\begin{cases} 
			\alpha {x} & \text{if } x < 0 \\
			x          & \text{if } x \geq 0
		\end{cases}
	\end{equation}
	where $\alpha \not = 1$ can be positive, negative or $0$. If $\alpha = 0$, then \eqref{prelu} becomes ReLU activation function; if $\alpha = 0.01$, \eqref{prelu} is the Leaky ReLU activation function \cite{maas2013rectifier}, see Figure \ref{fig:prelua}. Let $q(x) = \sigma(x)-\sigma(x-c),$ $c>0$, we can plot its graph as shown in Figure \ref{fig:prelub}, where $c=1$. Furthermore, by subtracting $\alpha c$ from $q(x)$ and scaling the result by dividing it by the factor $1-\alpha$, we can define a function $a(x)$ as
	\begin{equation}\label{prelu2}
		a(x)  :=\frac{\sigma (x)-\sigma (x-c)-\alpha c}{1-\alpha} = 
		\begin{cases} 
			0   & \text{if } x < 0 \\
			x   & \text{if } 0\leq x<c\\
			c   & \text{if } x \geq c
		\end{cases},
	\end{equation}
	see Figure \ref{fig:preluc}.
	Also, we can define the basis function similar as ReLU activation function in 1-D:
	\begin{equation}\label{prelu3}
		\phi(x)  
		:= \frac{\sigma (x+h)-2\sigma(x)+\sigma (x-h)}{(1-\alpha)h}
		=	\begin{cases} 
			0 		           & \text{if } x < -h \\
			1+\frac{x}{h}      & \text{if } -h\leq x < 0\\
			1-\frac{x}{h}      & \text{if } 0\leq x < h \\
			0                  & \text{if } x \geq h
		\end{cases},
	\end{equation}
	where $h>0$, see Figure \ref{fig:prelud}.
	\begin{figure}
		\centering
		\subfigure[][]{%
			\label{fig:prelua}%
			\includegraphics[width=0.35\linewidth]{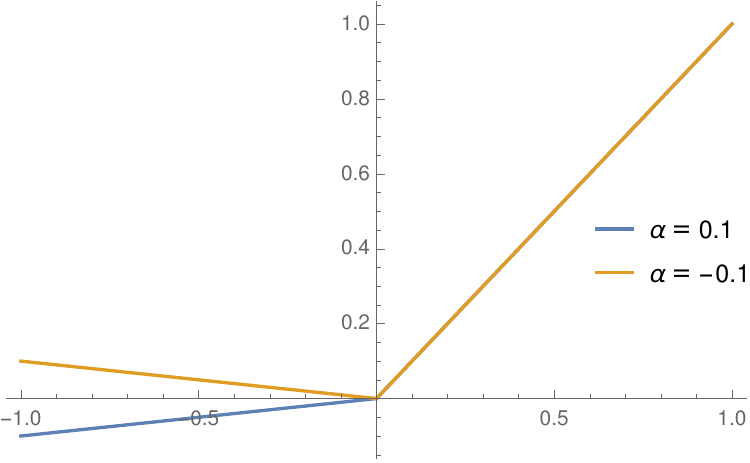}}%
		\hspace{8pt}%
		\subfigure[][]{%
			\label{fig:prelub}%
			\includegraphics[width=0.35\linewidth]{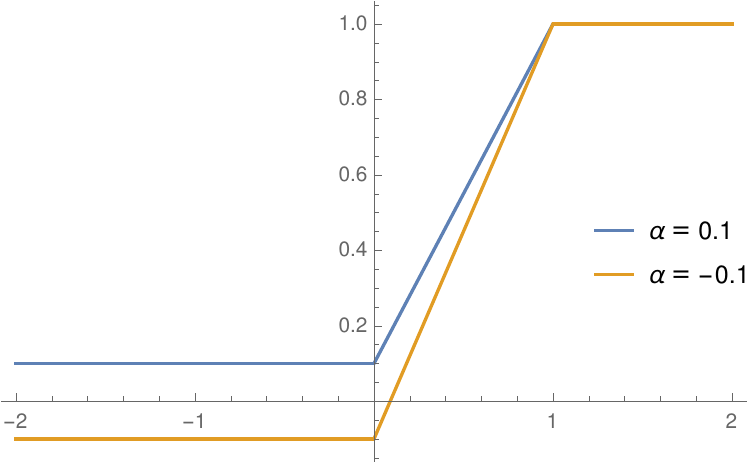}} \\
		\subfigure[][]{%
			\label{fig:preluc}%
			\includegraphics[width=0.35\linewidth]{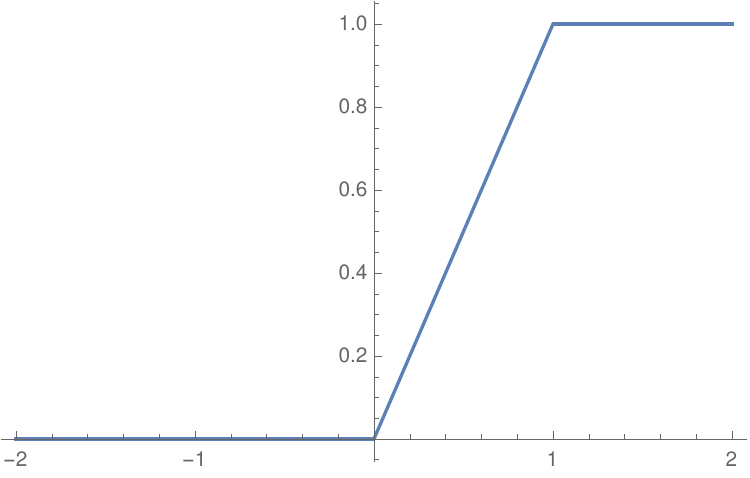}}%
		\hspace{8pt}%
		\subfigure[][]{%
			\label{fig:prelud}%
			\includegraphics[width=0.35\linewidth]{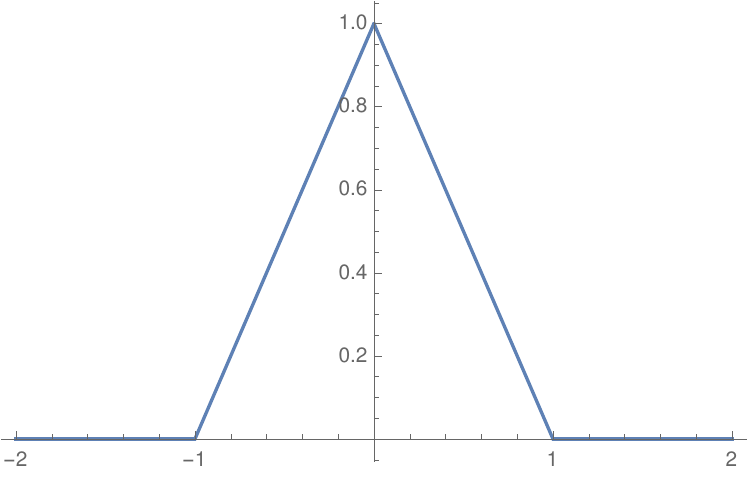}}%
		\caption[2d ``basis"]{
			\subref{fig:prelua} Graphs of two Parametric ReLU activation functions;
			\subref{fig:prelub} Graphs of $\sigma(x)-\sigma(x-1)$ with different $\alpha$;
			\subref{fig:preluc} Graph of $a(x)$ with $c=1$, $\alpha \not = 1$;
			\subref{fig:prelud} Graph of basis function $\phi(x)$ with $h=1,$ $\alpha \not = 1$.
		}
		\label{fig:prelu}%
	\end{figure}
	Here we have the building blocks $\phi(x)$ and $a(x)$ for constructing the optimal global solution of fully connected deep neural networks with Parametric ReLU activation functions. Compared with ReLU networks, the main differences in  formulating examples, which can fail the networks, lie in the second hidden layer.
	
	Similar to Section \ref{1de}, we consider a fully connected neural network with $2+\mathbb{M}$ hidden layers  and one output neuron with linear activation function. 
	We define the width of the first hidden layer as $H_1$, the width of the second hidden layer as $H_2$, and the width of the $(2+k)$th hidden layer as $H_{2+k}$, where $1 \leq k \leq \mathbb{M}$, $\mathbb{M}$ can be any positive integer. 
	
	\subsection{One Dimensional Binary Classification Problem}
	The binary classification problem is the same as in Section \ref{1dclass}. Let $x_1 = 0, x_N=1$ and $x_i$ distribute uniformly in $[0,1]$ with correct label, and the output of the neural network be $f_{2+\mathbb{M}}(x)$. 
	\begin{prop}\label{p411}
		If $H_1\geq 2N+1$, $H_2\geq 4$, and $H_{2+k}\geq 1$, $1 \leq k \leq \mathbb{M}$, there exist weights and biases such that the loss function is zero. Meanwhile, $\forall \epsilon > 0$ small enough, if $\text{dis}(x,\{x_i\}_{i=1}^{N}) > \epsilon$, then $f_{2+\mathbb{M}}(x) = 0$.
	\end{prop}
	\begin{proof}
		Let $2h = x_{i+1}-x_{i}$, and we define
		\begin{equation}\label{1dhatpl}
			\phi_i(x) = \phi(x-x_i),
		\end{equation}
		where $\phi(x)$ is from \eqref{prelu3}. For the first hidden layer, we need $2N+1$ neurons, each neuron has distinct input as: $x-x_i$ or $x-x_i-h$ or $x-x_i+h$, $i=1,2,\cdots,N$. From \eqref{prelu3} and \eqref{1dhatpl}, we can build basis functions $\phi_i(x), i=1,2,\cdots,N$. 
		
		Then we define the  linear combinations of first hidden layer's outputs as:
		\begin{equation}\label{1d1neuron-pl}
			I_1(x) := \sum_{i\ {\rm for} \ x_i\geq 0.5} \phi_i(x) -b_1,
		\end{equation}
		where $b_1$ is the bias.  Similarly, we have 
		\begin{equation}\label{1d2neuron-pl}
			I_2(x) := \sum_{i\ {\rm for} \ x_i< 0.5} \phi_i(x)  -b_2.
		\end{equation}  
		For the second hidden layer, we need at least 4 neurons. The inputs of the first and second neurons are: $I_1(x)$ and $I_1(x) - c$, where $c$ must be large enough, i.e., $c \geq \max_{x}(I_1(x))$. The inputs of the third and fourth neurons are: $I_2(x)$ and $I_2(x) - c$, where $c \geq \max_{x}(I_2(x))$.
		
		Let $b = b_1 = b_2 \in [0,1)$, then let the output of the second hidden layer be:
		$$
		f_2(x) := \frac{a(I_1(x))}{1-b}+(-1)\frac{a(I_2(x))}{1-b},
		$$ 
		where $a(\cdot)$ is from \eqref{prelu2}. Take $I_1(x)$ for example. Since $c$ is greater than or equal to $I_1(x)$, we have
		\begin{equation}\label{prelu2I}
			a(I_1(x))   = 
			\begin{cases} 
				0   & \text{if } I_1(x) < 0 \\
				I_1(x)   & \text{if } I_1(x) \geq 0 
			\end{cases}
		\end{equation}
		where $x\in [0,1]$. So, if $c$ is big enough, $a(\cdot)$ plays the same role as the ReLU activation function in truncating the negative part of  $I_1(x)$. As $b$ approaches 1, the measure of the compact support of $f_2(x)$ will shrink to 0. However, $f_2(x_i)$, where $i=1,2,\cdots, N,$ has the correct label.
		
		If we add more hidden layers to the network, then we can pass $f_2(x)+\mathfrak{C},$ where $\mathfrak{C}>1$, to a neuron  in the later hidden layers. From the definition of the Parametric ReLU function, we have 
		$$\sigma(f_2(x)+\mathfrak{C}) = f_2(x)+\mathfrak{C}.$$ 
		The value of $f_2(x)+\mathfrak{C}$ won't be changed. Then, at the final step, we subtract $c_2$ from it, which gives us:
		$$
		f_{2+\mathbb{M}}(x) = f_{2}(x),
		$$
		where $f_{2+\mathbb{M}}(x)$ is the output of the deep neural network.
	\end{proof}

	\subsection{One Dimensional Function Approximation}
	Similar to Section \ref{1dfunc}, we use the Parametric ReLU network to approximate a continuous function $g(x)$, $x\in [0,1]$. The training set is $\{x_i,y_i\}_{i=1}^N,$ where $y_i=g(x_i)$. The cost functions are MSE  or MAE errors.  
	\begin{prop}
		If $H_1\geq N+2$, $H_2\geq 2N$, and $H_{2+k}\geq 1$, $1 \leq k \leq \mathbb{M}$, there exist weights and biases such that the loss function is zero. Meanwhile, $\forall \epsilon > 0$ small enough, if $\text{dis}(x,\{x_i\}_{i=1}^{N}) > \epsilon$, then $f_{2+\mathbb{M}}(x) = 0$, indicating if $g(x)\not= 0$, the approximation is poor.
	\end{prop}
	\begin{proof}
		Since $x_i \in [0,1]$ is distributed uniformly, with $x_1 = 0, x_N=1$, let $h = x_{i+1}-x_{i}$, the basis function is  
		\begin{equation}\label{1dhat2p}
			\phi_i(x) = \phi(x-x_i),
		\end{equation}
		where $\phi(x)$ is defined in \eqref{prelu3}. For the first hidden layer, we need $N+2$ neurons, with each neuron having a distinct input: $x-x_i$ or $x-x_i-h$ or $x-x_i+h$, $i=1,2,\cdots,N$. This allows us to build the basis functions $\phi_i(x)$. 
		Then we need $2N$ neurons in the second hidden layer. Following the construction in Proposition \ref{p411}, the final output is:
		$$
		f_{2+\mathbb{M}}(x) := \sum_{i=1}^{N}y_i\, \frac{a(\phi_i(x)-b)}{1-b},
		$$ 
		where $a(\cdot)$ is denoted in \eqref{prelu2}, $c$ in $a(\cdot)$ should be large enough, and $b\in [0,1)$. So as $b\rightarrow 1$, the approximation will fail, even though $f_{2+\mathbb{M}}(x)$ is the global optimal solution. 
	\end{proof}
	
	\subsection{High Dimensional Problems}
	We employ the same settings and notations for high-dimensional binary classification and function approximation problems as in Section \ref{hd-relu}, except for the basis function $\phi(x)$ and the activation function. For Parametric ReLU, we have the following propositions.
	\begin{prop}
		For binary classification problem, if $H_1\geq 2dN+d$, $H_2\geq 4$,  and $H_{2+k}\geq 1$, $1 \leq k \leq \mathbb{M}$, there exist weights and biases such that the loss function is zero. Meanwhile, for any $\epsilon > 0$ small enough, if $\text{dis}({\bf x},\Xi_{\bf x}) > \epsilon$, then $f_{2+\mathbb{M}}({\bf x}) = 0$.
	\end{prop}
	\begin{proof}
		Assume that the second hidden layer has at least 4 neurons. The input of first neuron is:
		\begin{equation}\label{pdd1neuron}
			I_1({\bf x}) := 
			\sum_{x_{1,i_1}\geq 0.5} \phi_{1,i_1}
			+
			\sum_{j=2}^{d} 
			\sum_{ i_j=1}^{N}\phi_{j,i_j}
			-b_1
		\end{equation}
		where $b_1\in [d-1,d)$ is the bias, and $\phi_{j,i_j} = \phi(x_j -x_{j,i_j})$. 
		This ensures that $I_1(x)$ is the linear combination of outputs of the first hidden layer's neurons. The input of the second neuron is $I_1({\bf x}) -c$, where $c$ is large enough. Similarly, the input of the third neuron is:
		\begin{equation}\label{pdd2neuron}
			I_2({\bf x}) := 
			\sum_{x_{1,i_1}< 0.5} \phi_{1,i_1}
			+
			\sum_{j=2}^{d} 
			\sum_{ i_j=1}^{N}\phi_{j,i_j}
			-b_2
		\end{equation}
		where $b_2 \in [d-1,d)$. The input of the fourth neuron is $I_2({\bf x}) -c$, where $c\geq I_2(\bf x)$.
		
		Let $b = b_1=b_2$. Then the final output can be constructed as:
		$$
		f_{2+\mathbb{M}}({\bf x}) := 
		\frac{a(I_1({\bf x}))}{d-b}
		+(-1)\frac{a(I_2({\bf x}))}{d-b},
		$$ 
		where $a(\cdot)$ is defined in \eqref{prelu2}.
		The measure of the compact support for $f({\bf x})$ on $[0,1]^d$ decreases to 0 as $b\rightarrow d$, which concludes the proof. 
	\end{proof}
	\begin{prop}
		For function approximation problem, if $H_1\geq dN+2d$, $H_2\geq 2N^d$,  and $H_{2+k}\geq 1$, $1 \leq k \leq \mathbb{M}$, there exist weights and biases such that the loss function is zero. Meanwhile, for any $\epsilon > 0$ small enough, if $\text{dis}({\bf x},\Xi_{\bf x}) > \epsilon$, then $f_{2+\mathbb{M}}({\bf x}) = 0$,  which means if $g({\bf x})\not= 0$, the approximation is poor.
	\end{prop}
	\begin{proof}
		Let $h$ be $x_{j,i_j+1}-x_{j,i_j}$. We define the ``basis function" same as \eqref{ddhat2}.
		We need $2N^d$ neurons in second hidden layer, each pair is used to construct $a(\cdot)$, and the final output is:
		$$
		f_{2+\mathbb{M}}({\bf x}) := \sum_{i_1,i_2,\cdots,i_d=1}^{N}y_{i_1,i_2,\cdots,i_d}
		\frac{a(\Phi_{i_1,i_2,\cdots,i_d}({\bf x})-b)}
		{d-b},
		$$ 
		where $a(\cdot)$ is defined in \eqref{prelu2}, $c$ in $a(\cdot)$ should be large enough,
		$$
		y_{i_1,i_2,\cdots,i_d}=g(x_{1,i_1},x_{2,i_2},\cdots, x_{d,i_d}),
		$$
		and $b\in [d-1,d)$. So as $b\rightarrow d$, the approximation will fail, even though it is the global optimal solution. Deeper networks will fail in the same way as the 1-D case.
	\end{proof}
	
	\section{Constructions for Networks with Sigmoid Activation Functions}\label{sec5}
	In this section, we will show that to approximate functions using deep neural networks with sigmoid activation functions, there exist solutions that can be as close to the global optima as possible, but the approximation is still very poor. 
	
	Let $g({\bf x})$ be a continuous function, where ${\bf x}\in [0,1]^d.$ The training set $\Xi$ is defined in \eqref{Txy}. 
	The cost functions can be MSE or MAE errors. We follow the ideas proposed in Chapter 4 of \cite{nielsen2015neural}, which visually prove that a shallow neural network can compute any function. However, our purpose is to construct solutions that fit the training data very well but are only good approximations near the training points $\Xi_{\bf x}$. $\Xi_{\bf x}$ is defined in \eqref{Tx}. Then, we extend the results to deep networks by Theorem \ref{th2.1}.
	
	We begin by constructing the one-dimensional basis function $\phi(x)$, which is defined as
	\begin{equation}\label{phi_sig}
		\phi(x) = \frac{\sigma(Kx+1)-\sigma(Kx-1)}{\sigma(1)-\sigma(-1)}
	\end{equation}
	where $K>0$, $\sigma(x) = 1/(1+e^{-x})$ is the sigmoid function. From the definition, we know that $0<\phi(x)\leq 1$. It is symmetric and tends to a `spike' as $K$ increases, see Figure \ref{fig:sig}. 
	
	\begin{figure} 
		\centering
		\subfigure[][]{%
			\label{fig:siga}%
			\includegraphics[width=0.35\linewidth]{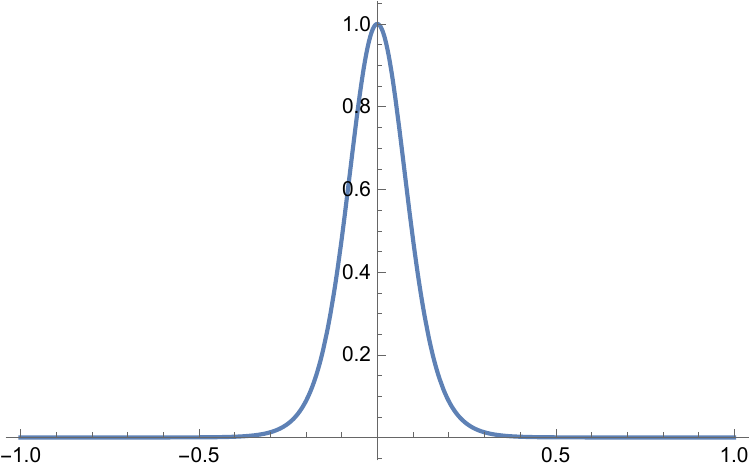}}%
		\hspace{8pt}%
		\subfigure[][]{%
			\label{fig:sigb}%
			\includegraphics[width=0.35\linewidth]{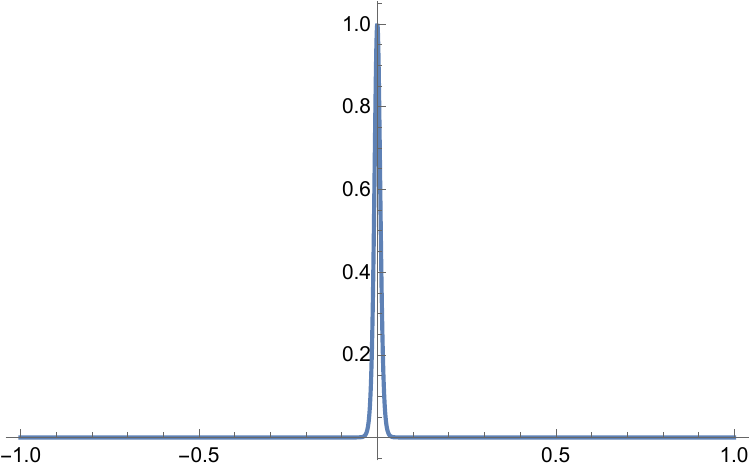}} 
		\caption[The basis fuction $\phi(x)$]{
			\subref{fig:siga} Graphs of $\phi(x)$ when $K=20$;
			\subref{fig:sigb} Graphs of $\phi(x)$ when $K=200$.
		}
		\label{fig:sig}%
	\end{figure}
	We then define a function $a(x)$, which is
	\begin{equation}\label{ax_sig}
		a(x) = \sigma(L x),
	\end{equation}
	where $L>0$. If we plug in a function $h(x)$ as $a(h(x))$, then for a large enough $L$, $a(\cdot)$ truncates the negative part of $h(x)$ and makes the positive part close to 1. Take $\phi(x)$ as an example, let $K=200$, we have Figure \ref{fig:sigL}. From Figures \ref{fig:sigLc} to \ref{fig:sigLd}, we observe that the ``compact support" of the truncated function $a((\phi(x)-b)/(1-b))$ shrinks as $b\rightarrow 1$. 
	\begin{figure} 
		\centering
		\subfigure[][]{%
			\label{fig:sigLc}%
			\includegraphics[width=0.35\linewidth]{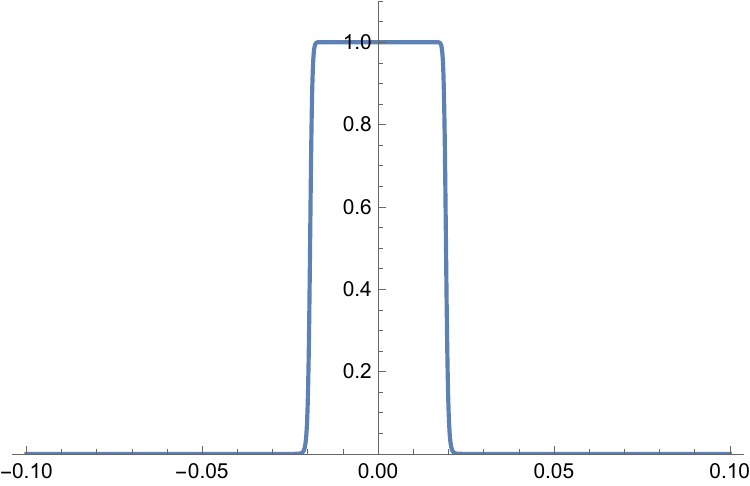}}%
		\hspace{8pt}%
		\subfigure[][]{%
			\label{fig:sigLd}%
			\includegraphics[width=0.35\linewidth]{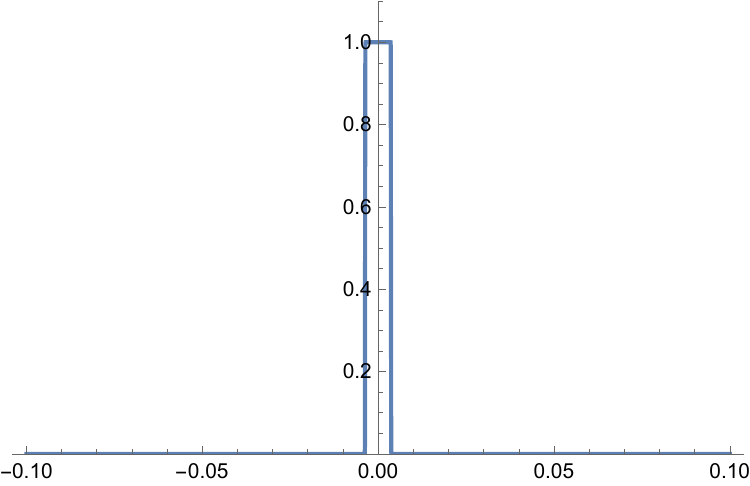}}%
		\caption[The basis fuction $\phi(x)$]{ Graphs of $a((\phi(x)-b)/(1-b))$ when $K=200$. 
			\subref{fig:sigLc} $L=150$, $b=0.1$;
			\subref{fig:sigLd} $L=150$, $b=0.9$.
		}
		\label{fig:sigL}%
	\end{figure}
	
	Similarly, we consider a fully connected neural network with $2+\mathbb{M}$ hidden layers and one output neuron that uses a linear activation function. The width of the first hidden layer is $H_1$,
	the width of the second hidden layer is $H_2$, and the width of the $(2 + k)$th hidden layer is $H_{2+k}$, where $1 \leq k \leq \mathbb{M}$. Here, $\mathbb{M}$ is a positive integer.
	\subsection{One-Dimensional Function Approximation}\label{1dsig_func}
	Suppose $d=1$. To approximate $g(x)$ on $[0,1]$ by the neural network, we define the training set $\{ x_i,y_i\}_{i=1}^{N},$ where $y_i=g(x_i)$.  
	\begin{prop}\label{prop_sig}
		If $H_1\geq 2N$, $H_2\geq N$, and $H_{2+k}\geq 1$, $1 \leq k \leq \mathbb{M}$, there exist weights and biases such that for any $\epsilon > 0$, the loss function is less than $\epsilon$. Meanwhile,  if $\text{dis}(x,\{x_i\}_{i=1}^{N}) > \epsilon$, then $|f_{2+\mathbb{M}}(x)| \leq \epsilon$, indicating if $g(x)\not= 0$, the approximation is poor.
	\end{prop}
	\begin{proof}
		First, we prove the result for the network with two hidden layers. For any $x_i \in [0,1]$, we denote
		\begin{equation}\label{phi_i_sig}
			\phi_i(x) = \phi(x-x_i)
		\end{equation}
		where $\phi(x)$ is defined in \eqref{phi_sig}.
		For the first hidden layer, we employ $2N$ neurons. Each pair of the neurons has inputs $K(x-x_i)+1$ and $K(x-x_i)-1$, where $K>0$ and $i=1,2,\cdots,N$. Then using the output of the first hidden layer, we can build basis functions $\phi_i(x)$, $i=1,2,\cdots,N$.
		Then we use $N$ neurons in second hidden layer and the output is:
		\begin{equation*}
			f_2(x) := \sum_{i=1}^{N}y_i\,  a\left( \frac{\phi_i(x)-b}{1-b}\right),
		\end{equation*}
		where $b\in [0,1)$. $L$ in $a(\cdot)$ is large enough.  So as $b\rightarrow 1$ or $K\rightarrow \infty$, the approximation will fail even though the loss function can be made as small as possible. Then, we can add $\mathbb{M}$ extra hidden layers. By Theorem \ref{th2.1} and equations \eqref{p1}, \eqref{pN}, and \eqref{fkn}, we can get $f_{2+\mathbb{M}}(x) \rightarrow f_2(x)$ by adjusting the parameters in equation \eqref{p1}.
	\end{proof}
	Next, we present examples to illustrate Proposition \ref{prop_sig}.
	\begin{example}\label{ex5.1}
		Let $g(x) = \sin(\pi x),$ where $x\in[0,1]$. The training set is: 
		$$\Big\{x_i,y_i\Big\}_{i=1}^9 \text{ where } x_i=\frac{i-1}{8},\ y_i=\sin(\pi x_i).$$
		So that we can construct the solution for a network with two hidden layers:
		\begin{equation}\label{f2sig_sin}
			f_2(x) = \sum_{i=1}^{9}\sin(\pi x_i)\,  \sigma\left( L
			\frac{\sigma(K(x-x_i)+1)-\sigma(K(x-x_i)-1) -b(\sigma(1)-\sigma(-1))}{(\sigma(1)-\sigma(-1))(1-b)}
			\right),
		\end{equation}
		where $\sigma(x) = 1/(1+e^{-x}).$
	\end{example}
	In equation \eqref{f2sig_sin}, for fixed $K$ and $L$, let $b$ vary from 0.2 to 0.995, this results in Figure \ref{fig:sigLE}. For fixed $L$ and $b$, let $K$ vary from 100 to 1000, we have Figure \ref{fig:sigLEK}.  So that as $b\rightarrow 1$ or $K\rightarrow \infty$, the approximation gets worse. 
	The extreme case is that we only have good approximations near the points $\{x_i\}_{i=1}^9$. For any $\epsilon >0$, if $b$ is close to 1 or $K$ is large enough, meanwhile $L$ is large enough, then $|f_2(x_i)-\sin(\pi x_i)|<\epsilon, i=1,2,\cdots,9,$ and the approximation is poor if 
	$\text{dis}(x,\{x_i\}_{i=1}^{9}) > \epsilon.$
	
	\begin{figure} 
		\centering
		\subfigure[][]{%
			\label{fig:sigLaE}%
			\includegraphics[width=0.35\linewidth]{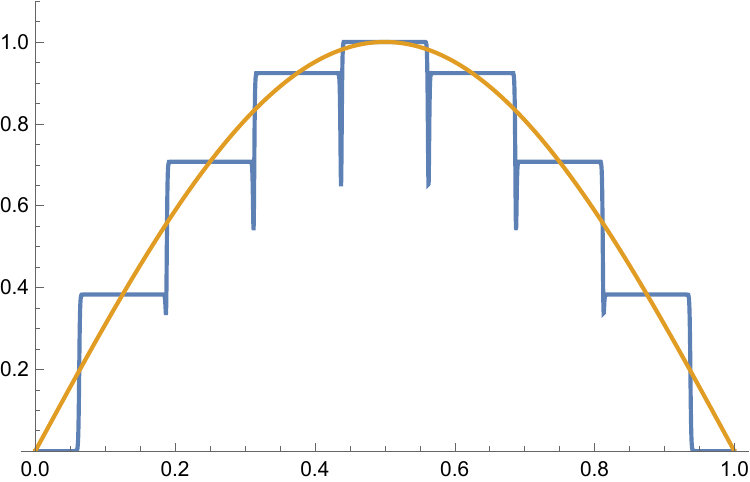}}%
		\hspace{8pt}%
		\subfigure[][]{%
			\label{fig:sigLbE}%
			\includegraphics[width=0.35\linewidth]{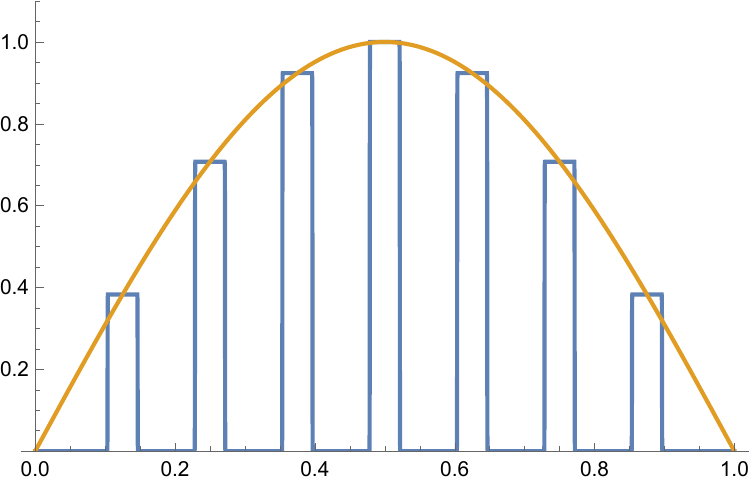}} 
		\subfigure[][]{%
			\label{fig:sigLcE}%
			\includegraphics[width=0.35\linewidth]{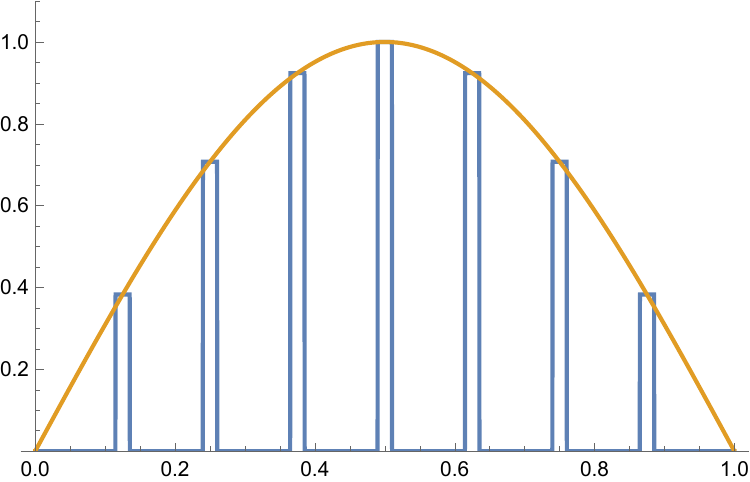}}%
		\hspace{8pt}%
		\subfigure[][]{%
			\label{fig:sigLdE}%
			\includegraphics[width=0.35\linewidth]{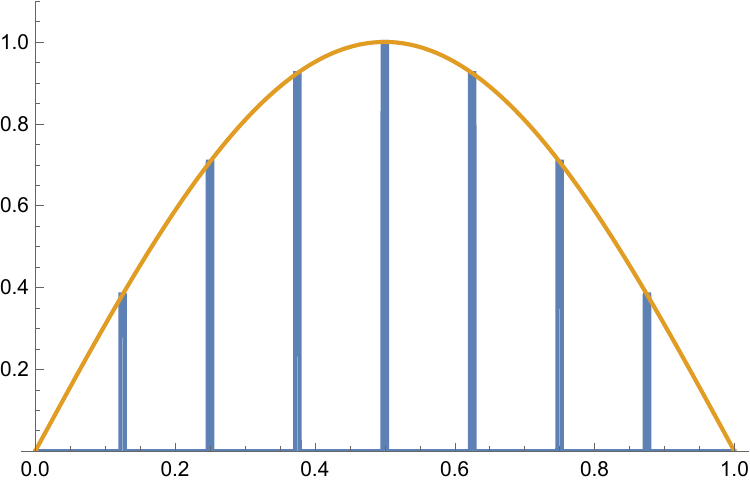}}%
		\caption[The basis fuction $\phi(x)$]{ Graphs of $f_2(x)$ in \eqref{f2sig_sin} (blue) and $\sin(\pi x)$ (yellow) when $K=50$, $L=150$. 
			\subref{fig:sigLaE} $b=0.2$;
			\subref{fig:sigLbE} $b=0.8$;
			\subref{fig:sigLcE} $b=0.95$;
			\subref{fig:sigLdE} $b=0.995$.
		}
		\label{fig:sigLE}%
	\end{figure}
	
	\begin{figure}
		\centering
		\subfigure[][]{%
			\label{fig:sigLaEK}%
			\includegraphics[width=0.35\linewidth]{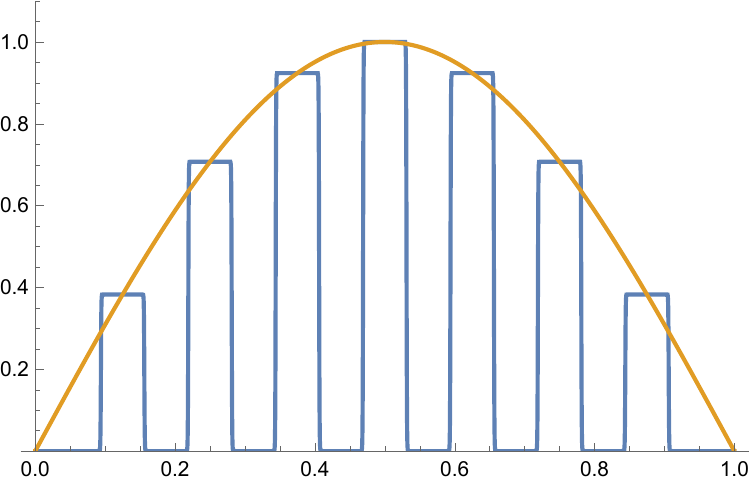}}%
		\hspace{8pt}%
		\subfigure[][]{%
			\label{fig:sigLbEK}%
			\includegraphics[width=0.35\linewidth]{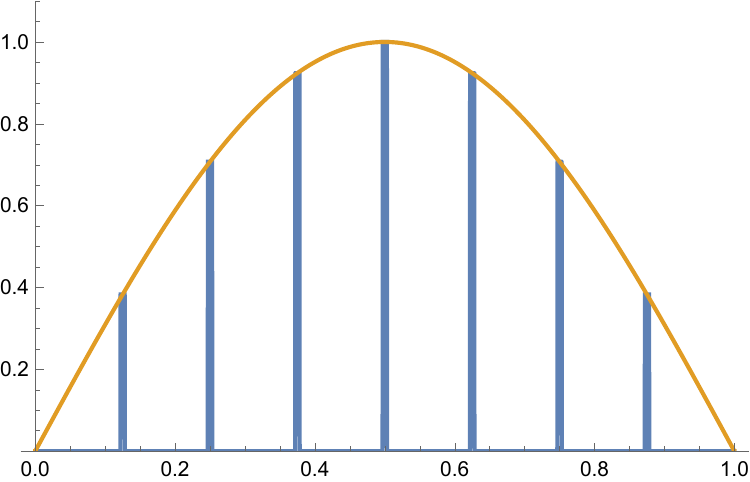}}
		\caption[The basis fuction $\phi(x)$]{ Graphs of \eqref{f2sig_sin} (blue) and $\sin(\pi x)$ (yellow) when $L=150$, $b=0.2$. 
			\subref{fig:sigLaEK} $K=100$;
			\subref{fig:sigLbEK} $K=1000$.
		}
		\label{fig:sigLEK}%
	\end{figure}
	
	\begin{example}\label{ex5.2}
		Based on Example \ref{ex5.1}, we can add $\mathbb{M}$ extra hidden layers to the network and construct 
		$f_{2+\mathbb{M}}(x)$, which is close to $f_2(x)$. Let $c=0$ and $a(\cdot)$ be $\sigma(x)$  in \eqref{p1}-\eqref{fkn}, we have $f_{2+\mathbb{M}}(x)$ as follows:
		\begin{align}
			&p_1 = \epsilon f_{2}(x), \label{p1_sig}\\
			&p_n = 4\,\sigma(p_{n-1})-2, \ n\geq 2, \\
			&f_{2+\mathbb{M}}(x) = \frac{4}{\epsilon}\sigma (p_\mathbb{M})-\frac{2}{\epsilon},\label{pN_sig}
		\end{align}
		where $\epsilon>0$ and $\sigma(x) = 1/(1+e^{-x})$.
	\end{example}
	Let $\mathbb{M}$ be 6, so we have six more hidden layers. The graphs of $f_{8}(x)$ and $f_2(x)$ are given in Figure \ref{fig:sigLaM}, which are very close to each other. From Figure \ref{fig:sigLbM} to Figure \ref{fig:sigLdM}, we can see that as $\epsilon$ in \eqref{p1_sig}-\eqref{pN_sig} decreases from 0.1 to 0.001, the error between $f_{8}(x)$ and $f_{2}(x)$ decreases. This example also verifies Theorem \ref{th2.1}.
	\begin{figure}
		\centering
		\subfigure[][]{%
			\label{fig:sigLaM}%
			\includegraphics[width=0.35\linewidth]{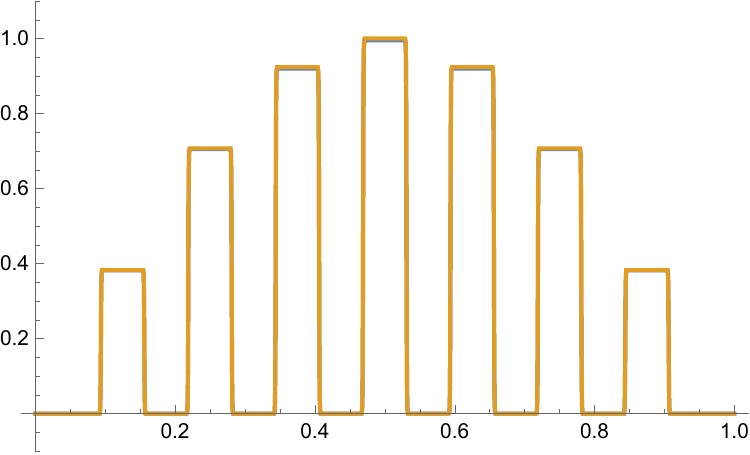}}%
		\hspace{8pt}%
		\subfigure[][]{%
			\label{fig:sigLbM}%
			\includegraphics[width=0.35\linewidth]{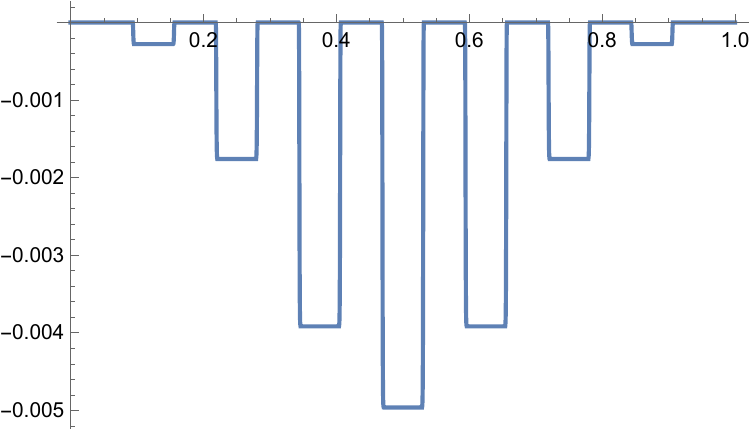}} %
		\subfigure[][]{%
			\label{fig:sigLcM}%
			\includegraphics[width=0.35\linewidth]{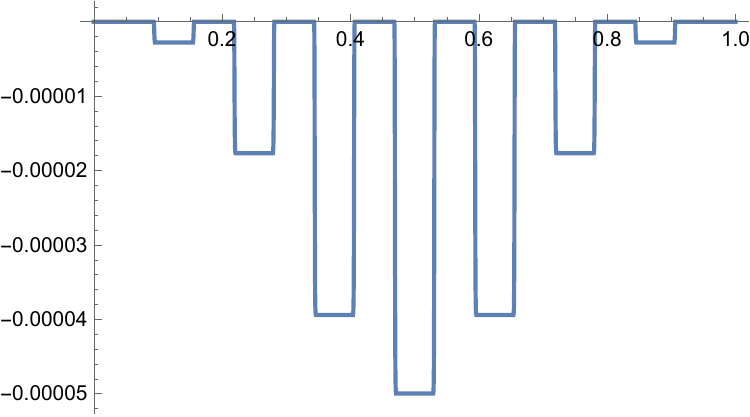}}%
		\hspace{8pt}%
		\subfigure[][]{%
			\label{fig:sigLdM}%
			\includegraphics[width=0.35\linewidth]{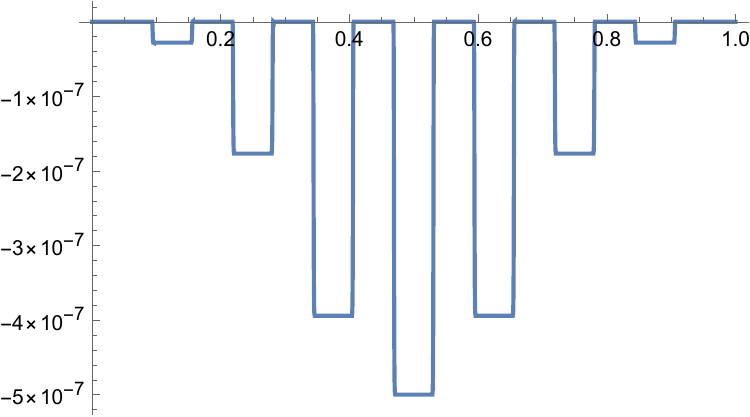}}%
		\caption[The basis fuction $\phi(x)$]{  
			\subref{fig:sigLaM} Graphs of $f_{8}(x)$ (blue) and $f_{2}(x)$ (yellow) when $\epsilon=0.1$;
			\subref{fig:sigLbM} Graph of $f_{8}(x)-f_{2}(x)$ when $\epsilon=0.1$;
			\subref{fig:sigLcM} Graph of $f_{8}(x)-f_{2}(x)$ when $\epsilon=0.01$;
			\subref{fig:sigLdM} Graph of $f_{8}(x)-f_{2}(x)$ when $\epsilon=0.001$.
		}
		\label{fig:sigLM}%
	\end{figure}
	
	\subsection{High-Dimensional Function Approximation}
	We employ the same settings and notations as in Section \ref{hd-relu} for function approximation; however, the network uses sigmoid activation functions. 
	
	\begin{prop}
		If $H_1\geq 2dN$, $H_2\geq N^d$,  and $H_{2+k}\geq 1$, $1 \leq k \leq \mathbb{M}$, there exist weights and biases such that for any $\epsilon > 0$, the loss function is less than $\epsilon$. Meanwhile,  if $\text{dis}({\bf x},\Xi_{\bf x}) > \epsilon$, then $|f_{2+\mathbb{M}}({\bf x})| \leq \epsilon$, indicating if $g({\bf x})\not= 0$, the approximation is poor.
	\end{prop}
	\begin{proof}
		We define $\phi_{i,i_j}$ same as \eqref{phi_j_ij}, but with $\phi(x)$ in \eqref{phi_sig},
		and the ``basis function": $\Phi_{i_1,i_2,\cdots,i_d}({\bf x})$ same as \eqref{ddhat2}.
		We need $2dN$ neurons in the first hidden layer. In the $j$th dimension, $j=1,2,\cdots, d$, we use $2N$ neurons, with each pair used to construct $\phi_{j,i_j}(x_j),$ where $x_j$ is the variable, $i_j=1,2,\cdots,N$. We need $N^d$ neurons in the second hidden layer. The output of the second hidden layer is:
		$$
		f_{2}({\bf x}) := \sum_{i_1,i_2,\cdots,i_d=1}^{N}y_{i_1,i_2,\cdots,i_d}\,
		a\left(\frac{\Phi_{i_1,i_2,\cdots,i_d}({\bf x})-b}
		{d-b}\right),
		$$ 
		where $a(\cdot)$ is defined in \eqref{ax_sig}, $L$ in $a(\cdot)$ should be large enough,
		$$
		y_{i_1,i_2,\cdots,i_d}=g(x_{1,i_1},x_{2,i_2},\cdots, x_{d,i_d}),
		$$
		and $b\in [d-1,d)$. 
		Then follow equations  \eqref{p1_sig}-\eqref{pN_sig}, we can get $f_{2+\mathbb{M}}({\bf x})$.
		So as $b\rightarrow d$ or $K\rightarrow \infty$, the approximation will fail, even though it is close to the global optimal solution. 
	\end{proof}
	
	\section{Conclusions}\label{sec6}
	We proposed a simple remedy to extend the universal approximation of shallow neural networks to any depth. The technique also works for vector-valued function approximation. However, if the dimension of the vector-valued function is $V$ (where $V$ is a positive integer), then for additional hidden layers, each layer should have at least $V$ neurons.
	The examples in Section \ref{sec3} to Section \ref{sec5} serve as extremely overfitting cases for fully connected deep neural networks. They are not practically useful but can help us understand the overfitting phenomenon and global optima theoretically. The example in Section \ref{sec_img} contradicts the common observation that to overfit, the number of parameters must significantly exceed the training data size. Binary classification examples can also be constructed for networks with sigmoid functions. Although at points not close to the training data, the output values of the network are not strictly zero, they are so small that they are not significant enough to be classified. We will extend the analysis in this paper to Recurrent Neural Networks and Convolutional Neural Networks.

	
	\bibliographystyle{plain}
	\bibliography{ref}

\end{document}